\documentclass[12pt]{article}
\usepackage{subfig}
\usepackage{wrapfig}
\usepackage{amsmath, amssymb, epsfig}
\usepackage[scaled]{helvet} 
\usepackage{fourier}
\usepackage{bm}
\usepackage{natbib}
\usepackage{color}
\usepackage{float}
\usepackage{rotating}
\usepackage{fullpage} 	
\usepackage[ruled]{algorithm2e}

\newcommand{\hide}[1]{}

\DeclareMathOperator*{\argmin}{arg min}
\DeclareMathOperator*{\argmax}{arg max}

\newcommand{\ignore}[1]{}

\newtheorem{thm}{Theorem}

\newtheorem{prop}{Proposition}

\newtheorem{defn}{Definition}

\newtheorem{rmk}{Remark}

\newcommand\remove[1]{}

\def\ud{\, \mathrm{d}}

\usepackage{palatino}

\usepackage{hyperref}

\usepackage[T1]{fontenc}
\usepackage[utf8]{inputenc}
\usepackage{tabularx,ragged2e,booktabs,caption}

\usepackage{csquotes}
\usepackage[ruled]{algorithm2e}
\usepackage{authblk}

\newtheorem{proof}{Proof}

\begin{document}
%

\title{Clustering Analysis on Locally Asymptotically Self-similar Processes with Known Number of Clusters}

\author{Qidi Peng\footnote{Institute of Mathematical Sciences, Claremont Graduate University, Claremont, CA 91711. Email: qidi.peng@cgu.edu.}, Nan Rao\footnote{Institute of Mathematical Sciences, Claremont Graduate University, Claremont, CA 91711. Email: nan.rao@cgu.edu.}, Ran Zhao\footnote{Institute of Mathematical Sceinces and Drucker School of Management, Claremont Graduate University, Claremont, CA 91711. Email: ran.zhao@cgu.edu.}}

\date{}
\maketitle

\begin{abstract}
We study the problems of clustering locally asymptotically self-similar stochastic processes, when the true number of clusters is priorly known. A new covariance-based dissimilarity measure is introduced, from which the so-called approximately asymptotically consistent clustering algorithms are obtained. In a simulation study, clustering data sampled from multifractional Brownian motions is performed to illustrate the approximated asymptotic consistency of the proposed algorithms.

\begin{flushleft}

\end{flushleft}
\end{abstract}

\section{Introduction}
\label{introduction}
Learning stochastic processes is an important area of machine learning, as there is a considerable amount of machine learning problems that involve time as a component \citep{Cotofrei2002,Harms2002,Jin2002,Jin2002indexing,Keogh2003}. Due to the nature of the time component, stochastic processes often possess path features \citep{Lonardi2002}. This additional information brought by the time component makes machine learning on stochastic processes more difficult to handle than machine learning on the other type objects. A number of new techniques are developed along with studying such machine learning problems \citep{Li1998,Bradley1998,Keogh2001,Java2002}. In this paper, we study a particular type of learning problems on stochastic processes: clustering. Among all machine learning tools, the cluster analysis is a common technique for unsupervised learning. It aims to detect the hidden patterns of a set of objects, through grouping the objects in the same cluster, if they are more similar to each other than to those in other clusters. Compared to other machine learning problems, clustering stochastic processes is more sensitive to dissimilarity measurement, which heavily depends on the objects' data type \citep{Gan2007data,Aghabozorgi2015,Sarda2017}. Clustering techniques may vary drastically depending on whether the objects' data types are vector, matrix or sequence, etc. Being a subset of clustering problems, clustering stochastic processes has received growing attention in diverse areas to discover patterns of data indexed by \enquote{time}. The stochastic process is a common type of dynamic data that naturally arises in many different scenarios. They have been broadly explored in information technology \citep{Slonim2005,Jain1999},  signal and image processing \citep{Yairi2001,Honda2002,Guha2003,Truppel2003,Rubinstein2013}, geology \citep{Juo2001,Harms2002data}, biology and medical research \citep{Bar2002,Damian2007,Zhao2014,Jaaskinen2014}, robotics \citep{Oates1999}, and finance \citep{Mantegna1999,Gavrilov2000,Tino2000,Fu2001,Pavlidis2006,Bastos2014,Ieva2016}, etc. Unlike random vector type data, stochastic processes type data are sampled from processes' distributions, which possess not only finite dimensional distribution features but also infinite-length paths features, such as stationarity, ergodicity, seasonality and Markov property.

In the problem of clustering stochastic processes, dimensionality of a process related to time is extremely high, compared to dimensionality of commonly observed data. Therefore new challenges arise if one applies the conventional approaches for cluster static data on stochastic processes. usually become computationally forbidding \citep{Sarda2017}. We list at least three issues that may occur when performing these clustering algorithms.
\begin{description}
\item[(1)] Thanks to the large data volume and high dimension, the conventional approaches for clustering static data (i.e. data who do not change with time) usually become computationally forbidding \citep{Sarda2017}.
\item[(2)] Even when the sample observations are of relatively low dimension, the conventional clustering approaches might suffer from over-fitting issues. For instance, clustering stationary (or seasonal) processes using the $K$-means approach with Euclidean distance between the sample paths, will result in large prediction mis-clustering errors in model validation. This is because, if one does not take into account the stationarity (or seasonality) of the process, it is then unable to reduce the noise on the stationary mean and covariances (or the period) along that process.
\item[(3)] More interestingly, \cite{Keogh2005} surprisingly proved that, without any other information on the observed stochastic processes, clustering their sample paths (subsequences) is \enquote{meaningless}, i.e. the output does not depend on the input data.
\end{description}

To overcome the above issues, we assume some key path features of the observed stochastic processes are known in this framework (see Assumption $(\mathcal A)$ in the next section). This is not a strong constrain, for the reason that in many fields it is proved that the observed paths are sampled from well-known process distributions. For example, dynamics in financial markets (equity returns and interest rates) can be described based on Geometric Brownian motions (GBm); long-term dependent or self-similar phenomena are often modeled by fractional Brownian motions (fBm). Contrary to the conventional clustering approaches, clustering based on the paths features of the processes largely removes the noise by capturing the observations' paths features. Therefore, a nice dissimilarity measure should be the one that well characterizes the paths features. In this context, ``nice'' refers to the property that the computational complexity and the prediction errors caused by the over-fitting issues are expected to be largely reduced. Moreover, subject to known paths features, consistency of the clustering algorithm \citep{Khaleghi2012,khaleghi2016} may be obtained. Among all the stochastic process features, we focus on characterizing the property of ergodicity in this paper. However similar analysis can be made for other patterns of process features such as seasonality, Markov property and martingale property.

Ergodicity \citep{K85} is a very typical feature possessed by a number of well-known processes, which are applied to financial time series analysis. It is tightly related to other process features such as stationarity, long-term memory and self-similarity \citep{Grazzini2012,Samorodnitsky2004}. \cite{khaleghi2016} and \cite{PRZ19} showed that both distribution ergodicity and covariance ergodicity lead to obtaining an asymptotically consistent clustering algorithms for clustering processes. In this paper, we step further to relax the condition of ergodicity to the \enquote{local asymptotic ergodicity} \citep{Boufoussi2008} and obtain the so-called \enquote{approximately asymptotically consistent algorithms} for clustering processes having such path property. This setting presents such a large class of processes that includes the well-known L\'evy processes, some self-similar processes and some multifractional processes \citep{Boufoussi2008}.

Each clustering stochastic processes problem involves handling data, defining cluster, measuring dissimilarities and finding groups efficiently, therefore we organize the paper as follows. Section \ref{main_framework} is devoted to introducing a class of locally asymptotically self-similar processes to which our clustering approaches apply. In Section \ref{Preliminary_Results}, a covariance-based dissimilarity measure is suggested and in Section \ref{sec::algo_consist} the approximately asymptotically consistent algorithms for clustering both offline and online datasets are designed. A simulation study is performed in Section \ref{sec::exper_results}, where the algorithms are applied to cluster multifractional Brownian motions (mBm), an excellent representative of the class of locally asymptotically self-similar processes. In Section \ref{sec:real_world}, we perform cluster analysis over the real world global financial market data through applying our clustering algorithms and provide economic implications. \ref{conclusion} concludes our findings.

\section{A Class of Locally Asymptotically Self-similar Processes}
\label{main_framework}
Self-similarity is a process (path) feature. Self-similar processes are a class of processes that are invariant in distribution under suitable scaling of time \citep{Samorodnitsky1994,Embrechtsu2000,Embrechts2002}. These processes have been used to successfully model various time-scaling random phenomena observed in high frequency data, especially in the geological data and financial data.
\begin{defn}[Self-similar process]
\label{self-similar}
A stochastic process $\{Y_t^{(H)}\}_{t\ge0}$ (here the time indexes set is not necessarily continuous) is self-similar with self-similarity index $H\in(0,1)$ if, for all $n\in\mathbb N:=\{1,2,\ldots\}$, all $t_1,\ldots,t_n\ge 0$  and all $c>0$,
\begin{equation}
\label{def:self-similar}
\Big(Y_{ct_1}^{(H)},\ldots,Y_{ct_n}^{(H)}\Big)\stackrel{\mbox{\text{law}}}{=}\Big(c^{H}Y_{t_1}^{(H)},\ldots,c^{H}Y_{t_n}^{(H)}\Big),
\end{equation}
where $\stackrel{\mbox{law}}{=}$ denotes the equality in joint probability distribution of two random vectors.
\end{defn}
When $\mathbb E|Y_t^{(H)}|<+\infty$ for $t\ge0$, it follows from (\ref{def:self-similar}) that for any $t\ge0$,
\begin{equation}
\label{E_Y_t}
\mathbb E\left(Y_t^{(H)}\right)=c^H\mathbb E\left(Y_{t/c}^{(H)}\right),~\mbox{for all}~ c>0,
\end{equation}
therefore taking $t=0$ at both hand sides of (\ref{E_Y_t}) yields
\begin{equation}
\label{mean-0}
\mathbb E\left(Y_0^{(H)}\right)=0.
\end{equation}
Self-similar processes are generally not distribution stationary but their increment processes can be distribution stationary (any finite subset's joint distribution is invariant subject to time shift) or covariance stationary (its mean and covariance structure exist and are invariant subject to time shift). From now on we restrict our setting to stochastically continuous-time self-similar processes only \citep{Embrechts2002}. i.e., a process $\{X_t\}_{t\ge0}$ is stochastically continuous at $t_0\ge0$ if
$$
\mathbb P(|X(t_0+h)-X(t_0)|>\varepsilon)\xrightarrow[h\to0^+]{}0,~\mbox{for any}~\varepsilon>0.
$$
This assumption is weaker than the almost sure continuity. The process $\{X_t\}_{t\ge0}$ is called (stochastically) continuous-time over $[0,+\infty)$ if it is continuous at each $t\ge0$. For $u>0$, we call $\{Y(t)\}_t=\{X(t+u)-X(t)\}_t$ the increment process (or simply the increments) of $\{X(t)\}_t$. If a continuous-time self-similar process' all increment processes are covariance stationary, its covariance structure can be explicitly given as below:
\begin{thm}
\label{cov:self}
Let $\big\{X_t^{(H)}\big\}_{t\ge0}$ be a self-similar process with index $H\in(0,1)$ and with covariance stationary increments. Then
\begin{equation}
\label{zero-mean}
\mathbb E\left(X_t^{(H)}\right)=0,~\mbox{for all}~t\ge0,
\end{equation}
and
\begin{equation}
\label{self-similar-cov}
\mathbb Cov\left(X_s^{(H)}, X_t^{(H)} \right)=\frac{Var(X_1^{(H)})}{2}\left(|s|^{2H}+|t|^{2H}-|s-t|^{2H}\right),~\mbox{for any $s,t\ge0$}.
\end{equation}
\end{thm}
Theorem \ref{cov:self} can be obtained by replacing the distribution stationary increments in Theorem 1.2 in \cite{Embrechtsu2000} with covariance stationary increments. We briefly provide the proof below.
\begin{proof}
We first prove (\ref{zero-mean}). On one hand, by using the fact that the increments of $\{X_t^{(H)}\}_t$ are covariance stationary and (\ref{mean-0}), we have
\begin{equation}
\label{part1:mean}
\mathbb E\left(X_{mt}^{(H)}\right)=\mathbb E\left(\sum_{k=0}^{m-1}\left(X_{(k+1)t}^{(H)}-X_{kt}^{(H)}\right)+X_0^{(H)}\right)=m\mathbb E\left(X_t^{(H)}\right),~\mbox{for all $m\in\mathbb N$}.
\end{equation}
On the other hand, since $\{X_t^{(H)}\}_t$ is self-similar, we have
\begin{equation}
\label{part2:mean}
\mathbb E\left(X_{mt}^{(H)}\right)=m^H\mathbb E\left(X_t^{(H)}\right),~\mbox{for all}~m\in\mathbb N.
\end{equation}
Putting together (\ref{part1:mean}),  (\ref{part2:mean}) and the fact that $H<1$, we necessarily have $\mathbb E(X_t^{(H)})=0$ for all $t\ge0$. (\ref{zero-mean}) is proved.

For proving (\ref{self-similar-cov}) we first observe that, for $s,t\ge0$,
\begin{equation}
\label{part1:cov}
\mathbb E\left(X_s^{(H)}X_t^{(H)}\right)=\frac{1}{2}\left(\mathbb E\left(X_s^{(H)}\right)^2+\mathbb E\left(X_t^{(H)}\right)^2-\mathbb E\left(X_s^{(H)}-X_t^{(H)}\right)^2\right).
\end{equation}
Next we can see from the facts that $\{X_t^{(H)}\}_t$ is self-similar with index $H$, that its increments are covariance stationary and (\ref{zero-mean}), that
\begin{eqnarray}
\label{part2:cov}
&&\mathbb E \left(X_s^{(H)}\right)^2=|s|^{2H}\mathbb E\left(X_1^{(H)}\right)^2=|s|^{2H}Var\left(X_1^{(H)}\right),~\mbox{for $s\ge0$};\nonumber\\
&&\mathbb E\left(X_s^{(H)}-X_t^{(H)}\right)^2=Var\left(X_s^{(H)}-X_t^{(H)}\right)\nonumber\\
&&=Var\left(X_{|s-t|}^{(H)}\right)=|s-t|^{2H}Var\left(X_1^{(H)}\right),~\mbox{for $s,t\ge0$}.
\end{eqnarray}
The covariance stationarity yields $Var(X_1^{(H)})<+\infty$. (\ref{self-similar-cov}) then follows from (\ref{part1:cov}) and (\ref{part2:cov}). Theorem \ref{cov:self} is thus proved.
\end{proof}
We highlight that, contrary to Theorem 1.2 in \cite{Embrechtsu2000}, the covariance stationary increment process of $\big\{X_t^{(H)}\big\}_{t}$ in Theorem \ref{cov:self} is not necessarily distribution stationary. This fact inspires us to relax the distribution stationarity of the processes to the covariance stationarity in the following Assumption ($\mathcal A$).  Below we introduce a natural extension of self-similar processes, the so-called locally asymptotically self-similar processes  \citep{Boufoussi2008,Falconer2002,FALCONER2003}.
\begin{defn}[Locally asymptotically self-similar process]
\label{locally_asymptotically_self_similar}
A continuous-time stochastic process $\big\{Z_t^{(H(t))}\big\}_{t\ge0}$ with its index $H(\bullet)$ being a continuous function valued in $(0,1)$, is called locally asymptotically self-similar, if for each $t\ge0$, there exists a non-degenerate self-similar process $\big\{Y_u^{(H(t))}\big\}_{u\ge0}$ with self-similarity index $H(t)$,  such that
\begin{equation}
\label{local_self}
\left\{\frac{Z_{t+\tau u}^{(H(t+\tau u))}-Z_t^{(H(t))}}{\tau^{H(t)}}\right\}_{u\ge0}\xrightarrow[\tau\to0^+]{\mbox{f.d.d.}}\left\{Y_u^{(H(t))}\right\}_{u\ge0},
\end{equation}
where the convergence $\xrightarrow[]{\mbox{\textit{f.d.d.}}}$ is in the sense of all the finite dimensional distributions.
\end{defn}
In (\ref{local_self}), $\{Y_u^{(H(t))}\}_u$ is called the \textit{tangent process} of $\{Z_t^{(H(t))}\}_t$ at $t$ \citep{Falconer2002,FALCONER2003}. Moreover, it is shown (see Theorem 3.8 in \cite{FALCONER2003}) that, if $\{Y_u^{(H(t))}\}_u$ is unique in law, it is then self-similar with index $H(t)$ and it has distribution stationary increments. Then the local asymptotic self-similarity generalizes the conventional self-similarity, in the sense that, any non-degenerate self-similar process with distribution stationary increments is locally asymptotically self-similar and its tangent process is itself. Further, in a weaker sense, it is not difficult to show the following:
\begin{prop}
\label{prop:tangent}
Let $\{Z_t^{(H)}\}_{t\ge0}$ be a continuous-time self-similar process with self-similarity index $H\in(0,1)$ and with covariance stationary increments. Then its tangent processes share equal mean and covariance functions.
\end{prop}
\begin{proof}
Since  $\{Z_t^{(H)}\}_{t\ge0}$ is locally asymptotically self-similar, by definition at each $t\ge0$ there exists a tangent process $\{Y_u^{(H)}\}_{u\ge0}$ such that
\begin{equation}
\label{assumption}
\left\{\frac{Z_{t+\tau u}^{(H)}-Z_t^{(H)}}{\tau^{H}}\right\}_{u\ge0}\xrightarrow[\tau\to0^+]{f.d.d.}\left\{Y_u^{(H)}\right\}_{u\ge0}.
\end{equation}
Next we show $\{Y_u^{(H)}\}_{u\ge0}$'s mean and covariance structure are uniquely determined.

Since $\{Z_t^{(H)}\}_{t\ge0}$ has covariance stationary increments, for any $u\ge0$, $\tau>0$, define the scaled increments
$$
Y_{u,\tau}^{(H)}:=\frac{Z_{t+\tau u}^{(H)}-Z_t^{(H)}}{\tau^{H}}.
$$
Again by the fact that $\{Z_t^{(H)}\}_{t\ge0}$ has covariance stationary increments, then using (\ref{zero-mean}) in Theorem \ref{cov:self} we obtain
\begin{equation}
\label{mean:increment}
\mathbb E\left (Y_{u,\tau}^{(H)}\right)=0,~\mbox{for all $u\ge0$, $\tau>0$},
\end{equation}
and by (\ref{self-similar-cov}) in Theorem \ref{cov:self}, we have for $u_1,u_2\ge0$ and $\tau>0$,
\begin{eqnarray}
\label{cov:increment}
\mathbb Cov\left(Y_{u_1,\tau}^{(H)},Y_{u_2,\tau}^{(H)}\right)&=&\tau^{-2H}\mathbb Cov\left(Z_{t+\tau u_1}^{(H)}-Z_t^{(H)},Z_{t+\tau u_2}^{(H)}-Z_t^{(H)}\right)\nonumber\\
&=&\frac{Var\left(Z_1^{(H)}\right)}{2}\left(|u_1|^{2H}+|u_2|^{2H}-|u_1-u_2|^{2H}\right),\nonumber\\
\end{eqnarray}
which is independent of $\tau$.

It follows from (\ref{assumption}), (\ref{mean:increment}) and (\ref{cov:increment}) that
\begin{eqnarray}
&&\mathbb E\left(Y_u^{(H)}\right)=\lim_{\tau\to0^+}\mathbb E\left(Y_{u,\tau}^{(H)}\right)=0,~\mbox{for all $u\ge0$};\nonumber\\
\label{cov_self_similar}
&&\mathbb Cov\left(Y_{u_1}^{(H)},Y_{u_2}^{(H)}\right)=\lim_{\tau\to0^+}\mathbb Cov\left(Y_{u_1,\tau}^{(H)},Y_{u_2,\tau}^{(H)}\right)\nonumber\\
&&=\frac{Var\left(Z_1^{(H)}\right)}{2}\left(|u_1|^{2H}+|u_2|^{2H}-|u_1-u_2|^{2H}\right),~\mbox{for~$u_1,u_2\ge0$.}
\end{eqnarray}
 This implies that all tangent processes of $\{Z_t^{(H)}\}_{t\ge0}$ possess zero-mean and equal covariance functions. By the way it is easy to derive from (\ref{cov_self_similar}) that these tangent processes have covariance stationary increments. Proposition \ref{prop:tangent} is proved.
\end{proof}
We remark from Proposition \ref{prop:tangent} that the tangent processes of $\{Z_t^{(H)}\}_{t\ge0}$ may not be unique in law, but their finite-dimensional subsets have unique first and second order moments.

Based on the above discussion, throughout this paper we assume that the observed dataset are sampled from a known number ($\kappa$) of continuous-time processes satisfying the following condition:
\newline

\noindent\textbf{Assumption $(\mathcal A)$:} The processes are \textit{locally asymptotically self-similar}; their tangent processes' increment processes are \textit{autocovariance ergodic}.
\newline

Here the autocovariance-ergodicity means that the sample autocovariance functions of the covariance stationary process converges in squared mean to the autocovaraince functions of the process, i.e., a zero-mean (that is the case for the tangent processes' increments) continuous-time process $\{X(t)\}_{t\ge0}$ is autocovariance ergodic if it is covariance stationary and satisfies
\begin{equation}
\label{mean-square:ergodic}
\frac{1}{T}\int_0^TX(t+\tau)X(t)\ud t\xrightarrow[T\to+\infty]{L^2(\mathbb P)}\mathbb E\left(X(u+\tau)X(u)\right),~\mbox{for all $u>0,\tau\ge0$},
\end{equation}
where $X_n\xrightarrow[n\to+\infty]{L^2(\mathbb P)}X$ denotes the mean squared convergence: $\mathbb E|X_n-X|^2\xrightarrow[n\to+\infty]{}0$.
Note that the above convergence (\ref{mean-square:ergodic}) yields
\begin{equation}
\label{P:ergodic}
\frac{1}{n-\tau-1}\sum_{k=1}^{n-\tau}X(k)X(k+\tau)\xrightarrow[n\to+\infty]{\mathbb P}\mathbb E(X(1)X(1+\tau)),~\mbox{for $\tau\in\mathbb N$}.
\end{equation}
Thus Assumption $(\mathcal A)$ says that the observed processes' tangent processes have covariance stationary increments. The well-known examples of locally asymptotically self-similar processes satisfying Assumption $(\mathcal A)$ are fractional Brownian motions and multifractional Brownian motions \citep{Mandelbrot1968, Peltier1995, Benassi1997}.

The assumption of covariance stationarity inspires us to introduce a covariance-based dissimilarity measure between the sample paths, in order to capture the level of differences between the two corresponding covariance stationary processes. Later we show that the assumption of autocovariance-ergodicity is sufficient for the clustering algorithms to be approximately asymptotically consistent.

\section{Clustering Stochastic Processes Satisfying Assumption $(\mathcal A)$}
\label{Preliminary_Results}
\subsection{Covariance-based Dissimilarity Measure between Autocovariance Ergodic Processes}
\label{X}
Let $Z$ be a process satisfying Assumption $(\mathcal A)$. Denote by $Y$ its tangent process (see (\ref{local_self})) and denote by $X$ an increment process of $Y$, i.e., there is some $u\ge0$ such that $X(t)=Y(t+u)-Y(u)$ for all $t\ge0$. Under Assumption $(\mathcal A)$, $X$ is autocovariance ergodic. Since we will show that clustering distinct $Z$'s is approximately asymptotically equivalent to clustering the corresponding increment processes $X$'s, then the dissimilarity measures of $Z$'s can be constructed based on those of the autocovariance ergodic processes $X$'s.

From (\ref{zero-mean}) we know that the autocovariance process $X$ is zero-mean.  Our first main result is then introduction to the following \textit{covariance-based dissimilarity measure} between autocovariance ergodic processes \citep{PRZ19}.
\begin{defn}
\label{definition1}
The covariance-based dissimilarity measure $d$ between the discrete-time stochastic processes $X^{(1)}$, $X^{(2)}$ (in fact $X^{(1)}$, $X^{(2)}$ denote two covariance structures, each class may contain different process distributions) is defined by
\begin{equation}
\label{def:d*}
d\left(X^{(1)},X^{(2)}\right):=\sum_{m,l = 1}^{+\infty} w_m w_l\rho\left(\mathbb Cov(X_{l\ldots l+m-1}^{(1)}),\mathbb Cov(X_{l\ldots l+m-1}^{(2)})\right),
\end{equation}
where:
\begin{itemize}
\item For any integers $l\ge 1$, $m\ge0$, $X_{l\ldots l+m-1}^{(1)}$ is the shortcut notation of the row vector $\Big(X_l^{(1)},\ldots,X_{l+m-1}^{(1)}\Big)$.
\item The distance $\rho$ between 2 equal-sized covariance matrices $M_1,M_2$ is defined to be the \textit{Frobenius norm} of $M_1-M_2$. Recall that for a matrix $A_{M\times N}$, its Frobenius norm is defined by
$$
\left\|A_{M\times N}\right\|_F:=\sqrt{\sum_{i=1}^M\sum_{j=1}^Na_{ij}^2},
$$
where for each $(i,j)\in \{1,\ldots,M\}\times \{1,\ldots,N\}$, $a_{ij}$ denotes the $(i,j)$-coefficient of $A_{M\times N}$.
\item The sequence of positive weights $\{w_j\}_{j\ge1}$ should be chosen such that $d\big(X^{(1)},X^{(2)}\big) < +\infty$, i.e., the series on the right-hand side of Eq. (\ref{def:d*}) is convergent. The choice of $\{w_j\}_j$ will be discussed in the forthcoming simulation study in Section \ref{sec::exper_results}.
\end{itemize}
\end{defn}
\begin{rmk}
\label{remark-1}
It is important to note that Definition \ref{definition1} only defines the dissimilarity measures between \textit{discrete-time} stochastic processes (time series). This is the most common object studied in the literature on clustering stochastic processes. Considering time series is sufficient for practical applications for at least 2 reasons. First, continuous-time path is not observable in practice. Secondly, any (stochastically) continuous-time process can be approximated by its discrete-time paths. In what follows we will mean sample path by a finite-length subsequence of discretized stochastic process.
\end{rmk}
Thanks to the autocovariance-ergodicity of the sample processes, the dissimilarity measure $d$ can be estimated by the empirical dissimilarity measure $\widehat d$ below:
\begin{defn}
Given two processes' discrete-time sample paths $\mathbf x_j=(X_1^{(j)},\ldots,X_{n_j}^{(j)})$ for $j=1,2$, let $n=\min\{n_1,n_2\}$, then the empirical covariance-based dissimilarity measure between $\mathbf x_1$ and $\mathbf x_2$ is given by
\begin{equation}
\label{dxx}
\widehat{d}(\mathbf x_{1},\mathbf x_{2}):=\sum_{m= 1}^{m_n} \sum_{l= 1}^{n-m+1}\!\! w_m w_l \rho\left (\nu(X^{(1)}_{l\ldots l+m-1}), \nu(X^{(2)}_{l\ldots l+m-1})\right),
\end{equation}
where:
\begin{itemize}
    \item $m_n$ ($\le n$) is the largest dimension of the covariance matrix considered by $\widehat d$; in this framework we take $m_n=\lfloor \log n\rfloor$, i.e. the floor number of $\log n$  \citep{khaleghi2016,PRZ19}.
    \item For $j=1,2$, $1\le l\le n$ and $m\le n-l+1$, $\nu(X_{l\ldots l+m-1}^{(j)})$ denotes  the empirical covariance matrix of the process $X^{(j)}$'s path $(X_l^{(j)},\ldots,X_{l+m-1}^{(j)})$, which is given below:
\begin{equation}
\label{nu}
\nu\left(X_{l\ldots l+m-1}^{(j)}\right):=\frac{\sum_{i = l} ^{n-m+1} (X_i^{(j)}~\ldots~X_{i+m-1}^{(j)})^T(X_i^{(j)}~\ldots~X_{i+m-1}^{(j)})}{n-m-l+2},
\end{equation}
where $(\bullet)^T$ denotes the transpose of a matrix.
\end{itemize}
\end{defn}
\begin{rmk} \label{remark0}
 Since $X$ is autocovariance ergodic, every empirical covariance matrix $\nu(X_{l\ldots l+m-1})$ is a  consistent estimator of the covariance matrix $\mathbb Cov(X_{l\ldots l+m-1})$ under Frobenius norm and in probability, i.e.,
\begin{equation}
\label{consistent:estimator}
\left\|\nu(X_{l\ldots l+m-1})-\mathbb Cov(X_{l\ldots l+m-1})\right\|_F\xrightarrow[n\to+\infty]{\mathbb P}0,~\mbox{for any $l\ge0$}.
\end{equation}
\end{rmk}
Further, the fact that both $d$ and $\widehat{d}$ satisfy the triangle inequalities implies that $\widehat d$ is a consistent estimator of $d$. The proof is quite similar to that of Lemma 1 in \cite{PRZ19}, except that in the former statement the convergence holds in probability. These ergodicity and triangle inequalities are the keys to demonstrate that our algorithms in the next section are approximately asymptotically consistent. We list them in the following remarks.
\begin{rmk} \label{remark1}
For every pair of paths
$$\mathbf{x_1}=\left(X_1^{(1)},\ldots,X_{n_1}^{(1)}\right), \quad \mbox{and} \quad \mathbf{x_2}=\left(X_1^{(2)},\ldots,X_{n_2}^{(2)}\right),
$$
sampled from two autocovariance ergodic processes $X^{(1)}$ and $X^{(2)}$ respectively, we have
\begin{align}
\label{limdxx}
\widehat{d}\left(\mathbf{x}_{1},\mathbf{x}_{2}\right) \xrightarrow[n_1,n_2 \rightarrow +\infty]{\mathbb P} d\left(X^{(1)},X^{(2)}\right),
\end{align}
and
\begin{align}
\label{limdxx1}
\widehat{d}\left(\mathbf{x}_{i},X^{(j)}\right) \xrightarrow[n_i \rightarrow \infty]{\mathbb P} d\left(X^{(1)},X^{(2)}\right), ~\mathrm{for}~ i,j\in\{1,2\},
\end{align}
where the dissimilarity measure $\widehat{d}\left(\mathbf{x}_{i},X^{(j)}\right)$ between the sample path $\mathbf{x}_{i}$ and the stochastic process $X^{(j)}$ is defined to be
$$
\widehat{d}\left(\mathbf{x}_{i},X^{(j)}\right):=\widehat{d}\left(\mathbf{x}_{i},\left(X_1^{(j)},\ldots,X_{n_i}^{(j)}\right)\right).
$$
\end{rmk}

\begin{rmk}
\label{remark2}
Thanks to their definitions, the triangle inequalities hold for the covariance-based dissimilarity measure $d$ in (\ref{def:d*}), as well as for its empirical estimates $\widehat{d}$ in (\ref{dxx}). Therefore for arbitrary processes $X^{(i)},~i = 1,2,3$ and arbitrary random vectors $\mathbf x_{i},~i=1,2,3$ we have
\begin{eqnarray*}
&&d\big(X^{(1)},X^{(2)}\big)  \leq d\big(X^{(1)},X^{(3)}\big) + d\big(X^{(3)},X^{(2)}\big),
\\
&&\widehat {d}(\mathbf{x}_{1},\mathbf{x}_{2})  \leq \widehat{d}(\mathbf{x}_{1},\mathbf{x}_{3}) + \widehat{d}(\mathbf{x}_{2},\mathbf{x}_{3}),
\\
&&\widehat{d}\big(\mathbf{x}_{1},X^{(1)}\big) \leq \widehat{d}\big(\mathbf{x}_{1},X^{(2)}\big) + d\big(X^{(1)},X^{(2)}\big).
\end{eqnarray*}
\end{rmk}
In the next section we define a proper covariance-based dissimilarity measure between locally asymptotically self-similar processes satisfying Assumption ($\mathcal A$), based on the dissimilarity measure $d$.
\subsection{Covariance-based Dissimilarity Measure between Locally Asymptotically Self-similar Processes}
Now under  Assumption ($\mathcal A$), we study the asymptotic relationship between the locally asymptotically self-similar process $\{Z_t^{(H(t))}\}_t$ in (\ref{local_self}) and its tangent process' increment process. The following result reveals the relationship between local asymptotic self-similarity and covariance stationarity.
\begin{prop}
\label{prop}
Let $\big\{Z_t^{(H(t))}\big\}_{t\ge0}$ be a locally asymptotically self-similar process satisfying Assumption ($\mathcal A$). For each $h>0$,
\begin{equation}
\label{local_self_1}
\left\{\frac{Z_{t+\tau (u+h)}^{(H(t+\tau (u+h)))}-Z_{t+\tau u}^{(H(t+\tau u))}}{\tau^{H(t)}}\right\}_{u\ge0}\xrightarrow[\tau\to0^+]{\mbox{f.d.d.}}\left\{X_{u,h}^{(H(t))}\right\}_{u\ge0},
\end{equation}
where $\big\{X_{u,h}^{(H(t))}\big\}_{u\ge0}:=\big\{Y_{u+h}^{(H(t))}-Y_{u}^{(H(t))}\big\}_{u\ge0}$ (see (\ref{local_self})) is an autocovariance ergodic process.
\end{prop}
\begin{proof}
Let's fix $h>0$ and pick any finite discrete time indexes set $T\subset[0,+\infty)$. Under Assumption ($\mathcal A$), the f.d.d. convergence (\ref{local_self}) holds. It then implies
\begin{equation}
\label{conv:law}
\left(\frac{Z_{t+\tau (u+h)}^{(H(t+\tau (u+h)))}-Z_t^{(H(t))}}{\tau^{H(t)}},\frac{Z_{t+\tau u}^{(H(t+\tau u))}-Z_t^{(H(t))}}{\tau^{H(t)}}\right)_{u\in T}\xrightarrow[\tau\to0^+]{\mbox{law}}\left(Y_{u+h}^{(H(t))},Y_{u}^{(H(t))}\right)_{u\in T},
\end{equation}
where we adopt the notation $(a_u,b_u)_{u\in\{u_1,\ldots,u_N\}}$ to denote the vector $$(a_{u_1},b_{u_1},a_{u_2},b_{u_2},\ldots,a_{u_N},b_{u_N}).$$
It follows from (\ref{conv:law}) and the continuous mapping theorem that
\begin{equation}
\label{conv:law_1}
\left(\frac{Z_{t+\tau (u+h)}^{(H(t+\tau (u+h)))}-Z_t^{(H(t))}}{\tau^{H(t)}}-\frac{Z_{t+\tau u}^{(H(t+\tau u))}-Z_t^{(H(t))}}{\tau^{H(t)}}\right)_{u\in T}\xrightarrow[\tau\to0^+]{\mbox{law}}\left(Y_{u+h}^{(H(t))}-Y_{u}^{(H(t))}\right)_{u\in T}.
\end{equation}
(\ref{local_self_1}) then results from (\ref{conv:law_1}) and the fact that the choice of $T$ is arbitrary. Under  Assumption ($\mathcal A$), $\big\{X_{u,h}^{(H(t))}\big\}_{u}:=\big\{Y_{u+h}^{(H(t))}-Y_{u}^{(H(t))}\big\}_{u}$ is autocovariance ergodic, hence Proposition \ref{prop} is proved.
\end{proof}
From a statistical point of view, the sequence at the left-hand side of (\ref{local_self_1}) can not straightforwardly serve to estimate the distribution of the right-hand side $\{X_{u,h}^{H(t)}\}_u$, because the functional index $H(\bullet)$ at the left-hand side is not observable in practice. To overcome this inconvenience we note that (\ref{local_self_1}) can be interpreted as: when $\tau$ is sufficiently small,
\begin{equation}
\label{local_self_2}
\left\{Z_{t+ \tau(u+h)}^{(H(t+\tau(u+h)))}-Z_{t+ \tau u}^{(H(t+\tau u))}\right\}_{u\in[0,Kh]}\stackrel{\mbox{f.d.d.}}{\approx}\left\{\tau^{H(t)}X_{u,h}^{(H(t))}\right\}_{u\in[0,Kh]},
\end{equation}
where $K$ is an arbitrary positive integer. Statistically, (\ref{local_self_2}) says that: given a discrete-time path $Z_{t_1}^{(H(t_1))},\ldots,Z_{t_n}^{(H(t_n))}$ with $t_i=ih\Delta t$ for each $i\in\{1,\ldots,n\}$, sampled from a locally asymptotically self-similar process $\{Z_t^{(H(t))}\}_t$, its \textit{localized increment paths} with time index around $t_i$, i.e.,
\begin{equation}
\label{increment:z}
\mathbf z^{(i)}:=\left(Z_{t_{i+1}}^{(H(t_{i+1}))}-Z_{t_i}^{(H(t_i))},\ldots,Z_{t_{i+1+K}}^{(H(t_{i+1+K}))}-Z_{t_{i+K}}^{(H(t_{i+K}))}\right),
\end{equation}
is \textit{approximately} distributed as an autocovariance ergodic increment process of the self-similar process $\left\{{\Delta t}^{H(t_i)}X_{u,h}^{(H(t_i))}\right\}_{u\in[0,Kh]}$. This fact drives us to define the empirical covariance-based dissimilarity measure between two paths of locally asymptotically self-similar processes $\mathbf z_1$ and $\mathbf z_2$ as follows:
\begin{equation}
\label{dissimilairty_local}
\widehat{d^{*}}(\mathbf z_1,\mathbf z_2):=\frac{1}{L}\sum_{i=1}^{L}\widehat{d}\left(\mathbf z_1^{(i)},\mathbf z_2^{(i)}\right),
\end{equation}
where:
\begin{itemize}
    \item $L\le n-K-1$ is a positive integer not depending on $K$;
    \item $\mathbf z_1^{(i)}$, $\mathbf z_2^{(i)}$ are the localized increment paths defined as in (\ref{increment:z}). Heuristically speaking, for $i=1,\ldots,n-K-1$, $\widehat{d}(\mathbf z_1^{(i)},\mathbf z_2^{(i)})$ computes the \enquote{distance} between the 2 covariance structures (of the increments of $\{Z_t^{H(t)}\}_t$) indexed by the time in the neighborhood of $t_i$, and $\widehat{d^{*}}(\mathbf z_1,\mathbf z_2)$ averages the above distances. It is worth noting that the value $K$ describes the \enquote{sample size} used to approximate each local distance $\hat d$. Therefore its value should be picked neither too large nor too small and it can depend on $n$.
    \end{itemize}

The following observation is straightforward.
\begin{rmk}
\label{remark3}
Based on the definition (\ref{dissimilairty_local}) and Remark \ref{remark0}, $\widehat{d^*}$ is also a (weakly) consistent estimator of $d$ (see Remark \ref{remark1}) and it also satisfies the triangle inequalities as in Remark \ref{remark2}.
\end{rmk}
Through employing the dissimilarity measure $\widehat{d^*}$ on locally asymptotically self-similar processes, we obtain the so-called \enquote{approximately asymptotically consistent algorithms}, which are introduced in the following section.

\section{Approximately Asymptotically Consistent Algorithms}
\label{sec::algo_consist}
\subsection{Offline and Online Algorithms}
Note that the covariance-based dissimilarity measure $\widehat{d^*}$ defined in (\ref{dissimilairty_local}) will aim to cluster covariance structures, not process distributions, therefore the ground truth of clustering should be based on covariance structures. We thus define the ground truth as follows \citep{PRZ19}.
\begin{defn}[Ground truth of covariance structures]
\label{ground-truth}
Let $G=\big\{G_1,\ldots,G_\kappa\big\}$ be a partitioning of $\mathbb N$ into $\kappa$ disjoint sets $G_k$, $k=1,\ldots,\kappa$, such that the means and covariance structures of $\mathbf x_i$, $i\in \mathbb N$ are identical, if and only if $i\in G_k$ for some $k=1,\ldots,\kappa$. Such $G$ is called ground truth of covariance structures. For $N\ge1$, we denote by $G|_{N}$ the restriction of $G$ to the first $N$ sequences:
$$
G|_{N}=\big\{G_k\cap \{1,\ldots,N\}:~k=1,\ldots,\kappa\big\}.
$$
\end{defn}
The processes $Z$ satisfying Assumption $(\mathcal A)$ are generally not covariance stationary, however their tangent processes' increments $X$ are covariance stationary. In view of (\ref{local_self_1}) and (\ref{local_self_2}), clustering these processes $Z$ is equivalent to clustering $X$, based on the covariance structure ground truth of the latter increments. Below we will introduce algorithms aiming to approximate the covariance structure ground truth of $X$.

Depending on how the information is collected, the processes clustering problems consist of dealing with two separate model settings: offline setting and online setting. In the offline setting, the sample size and each path length are time-independent. However, in the online setting, they may both grow with time.  As stated in \cite{khaleghi2016}, using the offline algorithm in the online setting by simply applying it to the entire data observed at every time step, does not result in an asymptotically consistent algorithm. As a result, we consider clustering offline and online datasets as 2 approaches and study them separately. Hence the approximated asymptotic consistency will be described in Theorem \ref{thm:offline} and Theorem \ref{thm:online} below, respectively for offline and online clustering algorithms. Our offline and online clustering algorithms below are obtained by replacing the dissimilarity measures in Algorithms 1 and 2 in \cite{PRZ19} with $\widehat{d^*}$.

For the offline setting, we cluster observed data using Algorithm \ref{algo::offline_known_k} below. It is a 2-point initialization centroid-based clustering approach. Under the distance $\widehat{d^*}$ defined in (\ref{dissimilairty_local}), the algorithm picks the farthest two points to be the first two cluster centers (Lines 1-2). Then iteratively, each next cluster center is chosen to be the point farthest to all the previously assigned cluster centers (Lines 3-5). Finally the algorithm assigns each remaining point to the nearest cluster (Lines 7-10).
\begin{algorithm}[h]\caption{Offline clustering} \label{algo::offline_known_k}

\LinesNumbered
\KwIn{\textbf{\textit{sample paths}} $S= \left \{ \mathbf{z}_1,\ldots,\mathbf{z}_N \right \}$; \textbf{\textit{number of clusters}} $\kappa$.}

$(c_1,c_2) \longleftarrow \argmax\limits_{(i,j)\in\{1,\ldots,N\}^2, i<j}\widehat{d^*}(\mathbf z_i,\mathbf z_j)$\;
$C_1 \longleftarrow \left \{ c_1 \right \}$; $C_2\longleftarrow\{c_2\}$\;
\For{$k = 3,\ldots,\kappa$}{
$c_k \longleftarrow \displaystyle\argmax_{i=1,\ldots,N}\displaystyle\min_{j = 1,\ldots,k-1} \widehat{d^*}(\mathbf{z}_i, \mathbf{z}_{c_j})$\;
}
\textbf{\textit{Assign the remaining points to the nearest centers}:}

\For{$i = 1,\ldots,N$}{
$k \longleftarrow \argmin\limits_{k\in\{1,\ldots,\kappa\}}\left\{\widehat{d^*}(\mathbf{z}_i, \mathbf{z}_j):~j \in C_k\right\}$;\\
$C_k \longleftarrow C_k \cup \left \{ i\right \}$\;
}
\KwOut{The $\kappa$ clusters $f(S,\kappa,\widehat{d^*})=\{C_1, \ldots, C_\kappa\}$.}
\end{algorithm}

The strategy for clustering online data is presented in Algorithm \ref{algo::online_known_k} as follows. At each time $t$, first update the collection of sample paths $S(t)$ (Lines 1-2). Then for each $j= \kappa,\ldots,N(t)$, use Algorithm \ref{algo::offline_known_k} to group the first $j$ paths in $S(t)$ into $\kappa$ clusters (Lines 6-7); for each cluster the center is selected as the point having the \textit{smallest} index among that cluster, and order these indexes increasingly (Line 8); calculate the minimum inter-cluster distance $\gamma_j$ and update the normalization factor $\eta$ (Line 9-11).  Finally, every sample path in $S(t)$ is assigned to the \enquote{nearest} cluster, based on the weighted combination of the distances (given in Line 15) between this observation and the candidate cluster centers obtained at each iteration on $j$ (Lines 14-17).

\begin{algorithm}[h]
\caption{Online clustering}
\label{algo::online_known_k}

\LinesNumbered
\KwIn{\textbf{\textit{Sample paths}} $\Big\{S(t)=\{\mathbf z_1^t,\ldots,\mathbf z_{N(t)}^t\}\Big\}_t$; \textbf{\textit{number of clusters}} $\kappa$; \textbf{\textit{a sequence of weights}} $\{\beta_j\}_j$.}

\For{$t = 1,\ldots,\infty$}{
\textbf{\textit{Obtain new paths:} $S(t) \longleftarrow \Big\{\mathbf{z}_1^t,\dots,\mathbf{z}_{N(t)}^t\Big \}$}\;
\textbf{\textit{Initialize the normalization factor}: $\eta \longleftarrow 0$}\;
\textbf{\textit{Initialize the $\kappa$ clusters}: $C_k(t) \longleftarrow \emptyset,~k = 1,\ldots,\kappa$}\;
\textbf{\textit{Generate }$N(t)-\kappa+1$ \textit{candidate cluster centers}:}

\For{$j = \kappa,\ldots,N(t)$}{
$\big \{C_1^j,\ldots, C_{\kappa}^j\big \} \longleftarrow\mbox{\textbf{Alg1}}\big( \big \{\mathbf{z}_1^t,\ldots, \mathbf{z}_j^t \big \}, \kappa, \widehat{d^*} \big)$\;
$(c_1^j,\ldots,c_k^j) \longleftarrow \mbox{sort}(\min \big \{ i\in C_k^j \big \}: k = 1,\ldots,\kappa$); $\backslash\backslash$ \textbf{\textit{Choose the smallest index to be the one for cluster center, and sort them increasingly}}\;
$\gamma_j \longleftarrow \min\limits_{k,k'\in \{1,\ldots,\kappa\},k\neq k'} \widehat{d^*}\big(\mathbf{z}_{c_k^j}^t,\mathbf{z}_{c_{k'}^j}^t\big)$\;
$w_j \longleftarrow \beta_j$; $\backslash\backslash$ \textbf{\textit{$w_j$ is the weight used in $\widehat{d^*}$}}\;
$\eta \longleftarrow \eta+w_j\gamma_j$; $\backslash\backslash$ \textbf{\textit{Update the  normalization factor}}\;}
\textbf{\textit{Assign each point to one of the $\kappa$ clusters}:}

\For{$i = 1,\ldots,N(t)$}{
$k \longleftarrow \argmin\limits_{k'\in \{1,\ldots,\kappa\}}\frac{1}{\eta} \sum\limits_{j=\kappa}^{N(t)}w_j \gamma_j \widehat{d^*}\big(\mathbf{z}_{i}^t,\mathbf{z}_{c_{k'}^j}^t\big)$\;
$C_k(t) \longleftarrow C_k(t)\cup \left\{ i \right\}$\;
}}
\KwOut{ \textit{The} $\kappa$ \textit{clusters} $f(S(t),\kappa,\widehat{d^*})=\left \{C_1(t), \dots, C_{\kappa}(t) \right \}$, $t=1,2,\ldots,\infty$.}
\end{algorithm}

\subsection{Computational Complexity and Consistency of the Algorithms}
We describe the computational complexity based on the number of computations of the distance $\rho$. For Algorithm \ref{algo::offline_known_k}, the 2-point initialization requires $N(N-1)/2$ times calculations of $\widehat{d^*}$. From (\ref{dissimilairty_local}) we see that each calculation of $\widehat{d^*}$ consists of $n_{\min}-K-1$ times computations of $\hat d$. By (\ref{dxx}), $\hat d$ can be obtained through computing $K-\log K+1$ times distances $\rho$. Therefore total number of computations of $\rho$ is not greater than $N(N-1)(n_{\min}-K-1)(K-\log K+1)/2$. For Algorithm \ref{algo::online_known_k}, since at each step $j\in\{\kappa,\ldots,N-\kappa+1\}$, Algorithm \ref{algo::offline_known_k} is run on $j$ observations, the total number of $\rho$'s computations is then less than $(n_{\min}-K-1)(K-\log K+1)\sum_{j=\kappa}^{N-\kappa+1}j(j-1)/2$. The computational complexity is acceptable in practice, and it is quite competitive to the existing algorithms for clustering stochastic processes.

Now we introduce the notion of approximately asymptotic consistency. Fix a positive integer $K$. Let $Z^{(1)}$, $Z^{(2)}$ be 2 locally asymptotically self-similar processes with respect functional indexes $H_1(\bullet)$, $H_2(\bullet)$. Also let
$$
\big\{\mathbf z_1^{(1)},\ldots,\mathbf z_1^{(n-K-1)}\big\} \quad \textrm{and} \quad \big\{\mathbf z_2^{(1)},\ldots,\mathbf z_2^{(n-K-1)}\big\},
$$
be respectively their sample paths $\mathbf z_1$, $\mathbf z_2$' increments, defined as in (\ref{increment:z}). For $j=1,2$,  we define the normalized increments by taking the following linear transformation:
\begin{equation}
\label{norm_increment}
\mathcal H\big(\mathbf z_j^{(i)}\big):=\frac{\mathbf z_j^{(i)}}{{\Delta t}^{H_j(t_i)}},~\mbox{for}~i=1,\ldots,n-K-1.
\end{equation}
Then using (\ref{local_self_1}) we obtain
\begin{equation}
\label{local_self_3}
\mathcal H\big(\mathbf z_j^{(i)}\big)\xrightarrow[\Delta t\to0]{\mbox{law}}\left(X_{0,j}^{(H_j(t_i))},X_{h,j}^{(H_j(t_i))},\ldots,X_{Kh,j}^{(H_j(t_i))}\right),
\end{equation}
where $\left(X_{0,j}^{(H_j(t_i))},X_{h,j}^{(H_j(t_i))},\ldots,X_{Kh,j}^{(H_j(t_i))}\right)$ denotes a discrete-time path of the increment of a self-similar process with self-similarity index $H_j(t_i)$.
Fix $L\ge1$. For each empirical dissimilarity measure $\widehat{d^*}(\mathbf z_1,\mathbf z_2)$, we correspondingly define
\begin{equation}
\label{dissimilairty_local_approx}
\widetilde{d^{*}}(\mathbf z_1,\mathbf z_2):=\frac{1}{L}\sum_{i=1}^{L}\widehat{d}\left(\mathcal H(\mathbf z_1^{(i)}),\mathcal H(\mathbf z_2^{(i)})\right).
\end{equation}
$\widetilde{d^{*}}(\mathbf z_1,\mathbf z_2)$ is another dissimilarity measure between $\mathbf z_1,\mathbf z_2$, which has a tight relationship to the distance $\widehat{d^*}$ between their tangent processes' increments. Indeed by using (\ref{local_self_3}) and the continuous mapping theorem, it is easy to derive the following result.
\begin{prop}
\label{prop_converge_delta}
For any $N$ independent sample paths $\mathbf z_1,\mathbf z_2,\ldots,\mathbf z_N$,
\begin{equation}
\label{converge:measure}
\left(\widetilde{d^{*}}(\mathbf z_i,\mathbf z_j)\right)_{i,j\in\{1,\ldots,N\}, i\neq j}\xrightarrow[\Delta t\to0]{\mbox{law}}\left(\widehat{d^{*}}(\mathbf x_i,\mathbf x_j)\right)_{i,j\in\{1,\ldots,N\},i\neq j},
\end{equation}
where $\mathbf x_1,\mathbf x_2$ are the increments of the tangent processes corresponding to $\mathbf z_1,\mathbf z_2$ respectively.
\end{prop}
In particular when $\Delta t=1$,  $\widetilde{d^*}(\mathbf z_1,\mathbf z_2)=\widehat{d^*}(\mathbf z_1,\mathbf z_2)$. In this sense  $\widetilde{d^*}(\mathbf z_1,\mathbf z_2)$ \enquote{approximates} $\widehat{d^*}$. $\widetilde{d^*}$ can not be observed in practice since the functional indexes of the locally asymptotically self-similar processes are supposed to be unknown. In what follows $\widetilde{d^*}$ only serves to define the approximate asymptotic consistency. In this notion \enquote{approximate} means the clustering locally asymptotically self-similar processes problem is \enquote{approximately} equivalent to the clustering their tangent processes problem.

Now we state the consistency theorems. Through Theorems \ref{thm:offline} and \ref{thm:online} below we show that Algorithms \ref{algo::offline_known_k} and \ref{algo::online_known_k} are both approximately asymptotically consistent. Their proofs are inspired by the ones in  \cite{khaleghi2016} and \cite{PRZ19}. However different from the consistent theorems in \cite{khaleghi2016} and \cite{PRZ19}, the convergences in Theorems \ref{thm:offline} and \ref{thm:online} have a weaker sense, which are in probability, not almost sure.
\begin{thm}
\label{thm:offline}
Under Assumption $(\mathcal A)$, Algorithm \ref{algo::offline_known_k} is approximately asymptotically consistent for clustering the offline sample paths $S=\{\mathbf z_1,\ldots,\mathbf z_N\}$. This means: if $\widehat{d^*}$ is replaced with $\widetilde{d^*}$ in Algorithm \ref{algo::offline_known_k}, then the output clusters converge to the covariance structure ground truths of the increments of the corresponding tangent processes $S'=\{\mathbf x_1,\ldots,\mathbf x_N\}$ in probability, as $\Delta t\to0$ and $n_{\min}:=\min\{n_1,\ldots,n_N\}\to+\infty$. More formally,
\begin{equation}
\label{converge:offline}
\lim_{n_{\min}\to+\infty}\lim_{\Delta t\to0}\mathbb P\left(f(S,\kappa,\widetilde{d^*})=G_{S'}\right)=1,
\end{equation}
where $f$ is given in Algorithm \ref{algo::offline_known_k} and $G_{S'}$ denotes the ground truths of the covariance structures that generate the set of paths $S'$.
\end{thm}
\begin{proof}
First, we show
\begin{equation}
\label{converge:t}
\mathbb P\left(f(S,\kappa,\widetilde{d^*})=G\right)\xrightarrow[\Delta t\to0]{}\mathbb P\left(f(S',\kappa,\widehat{d^*})=G\right),
\end{equation}
where $G=\{C_1,\ldots,C_\kappa\}$ denotes any $\kappa$-partition of $\{1,\ldots,N\}$. Since if the estimated clusters $f(S,\kappa,\widetilde{d^*})=G$, the  samples in $S$  that  are  generated  by  the same cluster in $G$ are  closer under $\widetilde{d^*}$  to  each  other  than  to  the  rest  of  the samples. Hence we can write
\begin{eqnarray}
\label{global}
&&\mathbb P\left(f(S,\kappa,\widetilde{d^*})=G\right)\nonumber\\
&&=\mathbb P\left(\bigcup_{\varepsilon>0}\left(\left\{\max_{\substack{l \in \{1,\ldots,\kappa\} \\i,j \in C_l }} \widetilde{d^*}(\mathbf{z}_i, \mathbf{z}_j)< \varepsilon\right\} \bigcap
\left\{\min_{\substack{k,k'\in\{1,\ldots,\kappa\},~ k\neq k'\\ i \in C_k,~j \in C_{k'}}} \widetilde{d^*}(\mathbf{z}_i, \mathbf{z}_j)>\varepsilon\right\}\right)\right).\nonumber\\
\end{eqnarray}
It follows from (\ref{global}) and (\ref{converge:measure}) that
\begin{eqnarray*}
\label{global1}
&&\lim_{\Delta t\to 0}\mathbb P\left(f(S,\kappa,\widetilde{d^*})=G\right)\nonumber\\
&&=\mathbb P\left(\bigcup_{\varepsilon>0}\left(\left\{\max_{\substack{l \in \{1,\ldots,\kappa\} \\i,j \in C_l }} \widehat{d^*}(\mathbf{x}_i, \mathbf{x}_j)< \varepsilon\right\} \bigcap
\left\{\min_{\substack{ i \in C_k,~j \in C_{k'} \\ k,k'\in\{1,\ldots,\kappa\}\\ k\neq k'}} \widehat{d^*}(\mathbf{x}_i, \mathbf{x}_j)>\varepsilon\right\}\right)\right)\nonumber\\
&&=\mathbb P\left(f(S',\kappa,\widehat{d^*})=G\right),
\end{eqnarray*}
which proves (\ref{converge:t}).

Next we show that Algorithm \ref{algo::offline_known_k} is asymptotically consistent on clustering $S'$ under $\widehat{d^*}$:
\begin{equation}
\label{cluster:x}
\mathbb P\left(f(S',\kappa,\widehat{d^*})=G_{S'}\right)\xrightarrow[n_{\min}\to\infty]{}1.
\end{equation}
Denote by
$$
G_{S'}:=\{G_1,\ldots,G_\kappa\}.
$$
Fix $\delta>0$. Let $n_{\min}$ denote the shortest path length in $S'$:
\begin{align*}
n_{\min}:= \min\left\{n_i:~i\in\{1,\ldots,N\}\right\}.
\end{align*}
Denote by $\delta_{\min}$ the minimum non-zero dissimilarity measure between the processes with different covariance structures:
\begin{equation}
\label{delta_min}
\delta_{\min}:= \min\left\{{d}\left(X^{(k)},X^{(k')}\right):~k,k'\in\{1,\ldots,\kappa\},~k\neq k'\right\}.
\end{equation}
Fix $\varepsilon \in (0, \delta_{\min} /4)$ and let $\delta>0$ be arbitrarily small. Since there are a finite number $N$ of samples, by Remark \ref{remark1} there is $n_0$ such that for $n_{min}>n_0$ we have
\begin{align}
\label{d:epsilon}
\mathbb P\left(\max_{\substack{l \in \{1,\ldots,\kappa\} \\i \in G_l \cap \left \{ 1,\ldots,N \right \}}} \widehat{d^*}\left(\mathbf{x}_i, X^{(l)}\right)>\varepsilon\right)<\delta.
\end{align}
On one hand, by applying the triangle inequalities (see Remark \ref{remark3}), we obtain
\begin{eqnarray}
\label{upperbound}
&&\max_{\substack{l \in \{1,\ldots,\kappa\} \\i,j \in G_l \cap \left \{ 1,\ldots,N \right \}}} \widehat{d^*}(\mathbf{x}_i, \mathbf{x}_j)\nonumber\\
&&\leq \max_{\substack{ l \in \{1,\ldots,\kappa\} \\i,j \in G_l \cap \left \{ 1,\ldots,N \right \}}} \widehat{d^*}\left(\mathbf{x}_i, X^{(l)}\right)+\max_{\substack{l \in \{1,\ldots,\kappa\} \\i,j \in G_l \cap \left \{ 1,\ldots,N \right \}}} \widehat{d^*}\left(\mathbf{x}_j, X^{(l)}\right)\nonumber\\
&&=2\max_{\substack{ l \in \{1,\ldots,\kappa\} \\i \in G_l \cap \left \{ 1,\ldots,N \right \}}} \widehat{d^*}\left(\mathbf{x}_i, X^{(l)}\right).
\end{eqnarray}
Then by  (\ref{upperbound}) and the fact that $2\varepsilon<\delta_{\min}/2$, the following inclusion holds:
\begin{equation}
\label{upperbound1}
\left\{\max_{\substack{l \in \{1,\ldots,\kappa\} \\i \in G_l \cap \left \{ 1,\ldots,N \right \}}} \widehat{d^*}\left(\mathbf{x}_i, X^{(l)}\right)\le\varepsilon\right\}\subset\left\{\max_{\substack{l \in \{1,\ldots,\kappa\} \\i,j \in G_l \cap \left \{ 1,\ldots,N \right \}}} \widehat{d^*}(\mathbf{x}_i, \mathbf{x}_j)\le 2\varepsilon<\frac{\delta_{\min}}{2}\right\}.
\end{equation}
On the other hand, by applying the triangle inequalities (see Remark \ref{remark3}) and the fact that $2\varepsilon<\delta_{\min}/2$, we also obtain: if
$$
\max_{\substack{l \in \{1,\ldots,\kappa\} \\i \in G_l \cap \left \{ 1,\ldots,N \right \}}} \widehat{d^*}\left(\mathbf{x}_i, X^{(l)}\right)\le\varepsilon,
$$
then
\begin{eqnarray*}
&&\min_{\substack{ k,k'\in\{1,\ldots,\kappa\},~k\neq k'\\i \in G_k \cap \left \{ 1,\ldots,N \right \} \\j \in G_{k'} \cap \left \{ 1,\ldots,N \right \}}} \widehat{d^*}(\mathbf{x}_i, \mathbf{x}_j) \nonumber\\
&&\geq \min_{\substack{ k,k'\in\{1,\ldots,\kappa\},~k\neq k'\\i \in G_k \cap \left \{ 1,\ldots,N \right \} \\j \in G_{k'} \cap \left \{ 1,\ldots,N \right \}}}\left\{ d\left(X^{(k)}, X^{(k')}\right) - \widehat{d^*}\left(\mathbf{x}_i, X^{(k)}\right)- \widehat{d^*}\left(\mathbf{x}_j, X^{(k')}\right)\right\}  \nonumber\\
&&\geq \delta_{\min}-2\varepsilon> \frac{\delta_{\min}}{2}.
\end{eqnarray*}
Equivalently,
\begin{equation}
\label{lowerbound}
\left\{\max_{\substack{l \in \{1,\ldots,\kappa\} \\i \in G_l \cap \left \{ 1,\ldots,N \right \}}} \widehat{d^*}\left(\mathbf{x}_i, X^{(l)}\right)\le\varepsilon\right\}\subset\left\{\min_{\substack{k,k'\in\{1,\ldots,\kappa\},~k\neq k'\\ i \in G_k \cap \left \{ 1,\ldots,N \right \} \\j \in G_{k'} \cap \left \{ 1,\ldots,N \right \}}} \widehat{d^*}(\mathbf{x}_i, \mathbf{x}_j)>\frac{\delta_{\min}}{2}\right\}.
\end{equation}
It follows from (\ref{upperbound1}), (\ref{lowerbound}) and (\ref{d:epsilon}) that for $n_{\min}>n_0$,
\begin{eqnarray}
\label{upperglobal}
&&\mathbb P\left(\left\{\max_{\substack{l \in \{1,\ldots,\kappa\} \\i,j \in G_l \cap \left \{ 1,\ldots,N \right \}}} \widehat{d^*}(\mathbf{x}_i, \mathbf{x}_j)\ge \frac{\delta_{\min}}{2}\right\}\bigcup\left\{
\min_{\substack{ k,k'\in\{1,\ldots,\kappa\},~k\neq k'\\i \in G_k \cap \left \{ 1,\ldots,N \right \} \\j \in G_{k'} \cap \left \{ 1,\ldots,N \right \}}} \widehat{d^*}(\mathbf{x}_i, \mathbf{x}_j)\le\frac{\delta_{\min}}{2}\right\}\right)\nonumber\\
&&\le\mathbb P\left(\max_{\substack{l \in \{1,\ldots,\kappa\} \\i \in G_l \cap \left \{ 1,\ldots,N \right \}}} \widehat{d^*}\left(\mathbf{x}_i, X^{(l)}\right)>\varepsilon\right)<\delta.
\end{eqnarray}
Note that (\ref{upperglobal}) is equivalent to
\begin{eqnarray*}
&&\mathbb P\left(\left\{\max_{\substack{l \in \{1,\ldots,\kappa\} \\i,j \in G_l \cap \left \{ 1,\ldots,N \right \}}} \widehat{d^*}(\mathbf{x}_i, \mathbf{x}_j)< \frac{\delta_{\min}}{2}\right\}\bigcap\left\{
\min_{\substack{ k,k'\in\{1,\ldots,\kappa\},~k\neq k'\\i \in G_k \cap \left \{ 1,\ldots,N \right \} \\j \in G_{k'} \cap \left \{ 1,\ldots,N \right \}}} \widehat{d^*}(\mathbf{x}_i, \mathbf{x}_j)>\frac{\delta_{\min}}{2}\right\}\right)\nonumber\\
&&\xrightarrow[n_{\min}\to+\infty]{}1.
\end{eqnarray*}
This tells that the sample paths in $S$ that are generated by the same covariance structures are closer to each other than to the rest of sample paths. Then by (\ref{upperglobal}), for $n_{\min}>n_0$, each sample path should be \enquote{close} enough to its cluster center, i.e.,
\begin{equation}
\label{center}
\mathbb P\left(\max_{i \in\{ 1,\ldots,N\}} \min_{k \in\{ 1,\ldots,\kappa-1\}} \widehat{d^*} (\mathbf{x}_i, \mathbf{x}_{c_k})\le\frac{\delta_{\min}}{2}\right)<\delta,
\end{equation}
where the $\kappa$ cluster centers' indexes $c_1,\ldots,c_\kappa$ are determined by Algorithm \ref{algo::offline_known_k} in the following way:
$$
(c_1,c_2) := \argmax_{i,j \in \{ 1,\ldots,N\},~i<j}\widehat{d^*} (\mathbf{x}_i, \mathbf{x}_{j}),
$$
and
$$c_k :=\argmax_{i \in \{1,\ldots,N\}} \displaystyle\min_{j \in \{ 1,\ldots,k-1\}} \widehat{d^*} (\mathbf{x}_i, \mathbf{x}_{c_j}),~k = 3,\ldots,\kappa.
$$
These $c_1,\ldots, c_{\kappa}$ are chosen to index sample paths generated by different process covariance structures. Then by (\ref{upperglobal}), each remaining sample path will be assigned to the cluster center corresponding to the sample path generated by the same process covariance structure. Finally (\ref{cluster:x}) results from (\ref{upperglobal}) and (\ref{center}); and (\ref{converge:offline}) is proved by combining (\ref{converge:t}) and (\ref{cluster:x}).
\end{proof}
Below we state the consistency theorem concerning the online clustering algorithm.
\begin{thm}
\label{thm:online}
Under Assumption $(\mathcal A)$, Algorithm \ref{algo::online_known_k} is approximately asymptotically consistent for clustering the online sample paths $S(t)=\{\mathbf z_1^t,\ldots,\mathbf z_{N(t)}^t\}$, $t=1,2,\ldots$. This means: if $\widehat{d^*}$ is replaced with $\widetilde{d^*}$ in Algorithm \ref{algo::online_known_k}, for any integer $N\ge 1$, the output clusters of the first $N$ paths in $S(t)$
$$
S(t)|_N:=\left\{\mathbf z_1^t,\ldots,\mathbf z_N^t\right\},
$$
converge to the covariance structure ground truths of the increments of the corresponding tangent processes $S'(t)|_N:=\{\mathbf x_1^t,\ldots,\mathbf x_N^t\}$ in probability, as $\Delta t\to0$ and $t\to+\infty$. In other words,
\begin{equation}
\label{converge:online}
\lim_{t\to+\infty}\lim_{\Delta t\to0}\mathbb P\left(f(S(t),\kappa,\widetilde{d^*})|_N=G_{S'(t)}|_N\right)=1,
\end{equation}
where  $f(S(t),\kappa,\widetilde{d^*})|_N$
denotes the clustering $f(S(t),\kappa,\widetilde{d^*})$ restricted to the first $N$ sample paths in $S(t)$. We also  recall that $G_{S'(t)}|_N$ is the restriction of $G_{S'(t)}$ to the first $N$ sample paths $\{\mathbf x_1^t,\ldots,\mathbf x_N^t\}$ in $S'(t)$ (see Definition \ref{ground-truth}).
\end{thm}

\begin{proof}
Let's fix $N\ge1$. First, similar to the derivation of (\ref{converge:t}) in the proof of Theorem \ref{thm:offline}, we can obtain
\begin{equation}
\label{converge:t_1}
\mathbb P\left(f(S(t),\kappa,\widetilde{d^*})|_N=G_{S'(t)}|_N\right)\xrightarrow[\Delta t\to0]{}\mathbb P\left(f(S'(t),\kappa,\widehat{d^*})|_N=G_{S'(t)}|_N\right).
\end{equation}
Then it remains to  prove
\begin{equation}
\label{cluster:x_1}
\mathbb P\left(f(S'(t),\kappa,\widehat{d^*})|_N=G_{S'(t)}|_N\right)\xrightarrow[t\to+\infty]{}1.
\end{equation}
In what follows we prove (\ref{cluster:x_1}).

Let $\delta>0$ be arbitrarily small. Fix $\varepsilon \in (0, \delta_{\min}/4)$, where  $\delta_{\min}$ is defined as in (\ref{delta_min}).

Denote by
\begin{equation}
\label{deltamax:kk'}
\delta_{\max}:= \max\left\{d\left(X^{(k)},X^{(k')}\right):~k,k'\in\{1,\ldots,\kappa\}\right\}.
\end{equation}
For $k\in\{1,\ldots,\kappa\}$, denote by $s_{k}$ the index of the first path in $S'(t)$ sampled from $X^{(k)}$, i.e.,
\begin{align}
\label{sk}
s_{k} := \min \left \{ i \in G_k \cap \{1,\ldots,N(t)\} \right \}.
\end{align}
Note that $s_k$ does not depend on $t$ but only on $k$.
Then denote by
\begin{align}
\label{m_t}
m := \max_{k \in \{1,\ldots,\kappa\}} s_{k}.
\end{align}
For $j\ge1$ denote by $S'(t)|_j$ the first $j$ sample paths contained in $S(t)$. Then from (\ref{m_t}) we see that, $m\le N(t)$ and $S'(t)|_{m}$ contains paths sampled from all $\kappa$ distinct processes (covariance structures).
By using the fact that $\sum_{j = 1}^{+\infty} w_j<+\infty$, we can find a fixed value $J\ge m$ such that
\begin{equation}
\label{bound_wJ}
\sum_{j = J+1}^{+\infty} w_j\le\varepsilon.
\end{equation}
Recall that in the online setting, the $i$th sample path's length $n_i(t)$ grows with time $t$. Therefore, by Remark \ref{remark3}, for every $j \in  \{1,\ldots,J\}$ there exists some $T_1(j)>0$ such that
\begin{equation}
\label{d:upperbound}
 \sup_{t\ge T_1(j)}\mathbb P\left(\max_{\substack{k \in \{1,\ldots,\kappa\} \\ i \in G_k \cap \left \{ 1,\ldots,j \right \}} } \widehat{d^*}\left(\mathbf{x}_i^t, X^{(k)}\right) >\varepsilon\right)<\delta.
\end{equation}
Since $J\ge m$, by Theorem \ref{thm:offline} for every $j \in \{m,\ldots,J\}$ there exists $T_2(j)>0$ such that $\mbox{Alg1}(S'(t)|_j, \kappa,\widehat{d^*})$ is asymptotically consistent for all $t \geq T_2(j)$. Since $N(t)$ is increasing as $t\to+\infty$, there is $T_3>0$ such that $N(t)>J$ for $t\ge T_3$. Let
\begin{align*}
T:= \max\left\{\max_{\substack{i\in\{1,2\}\\ j \in \{1,\ldots,J\}}} T_i(j),~T_3\right\}.
\end{align*}
From Algorithm \ref{algo::online_known_k} (Lines 9, 11) we see
\begin{equation}
\label{eta_gamma}
\eta^t : = \sum_{j=1}^{N(t)}w_j\gamma_j^t,\quad\mbox{with}\quad
\gamma_j^t:=\min_{\substack{k,k'\in \{1,\ldots,\kappa\}\\ k\neq k'}} \widehat{d^*}\left(\mathbf{x}_{c_k^j}^t,\mathbf{x}_{c_{k'}^j}^t\right).
\end{equation}
Below we provide upper bounds in probability of $\eta^t$ and $\gamma_j^t$.
\begin{description}
 \item[Upper bound of $\gamma_j^t$:]
Similar to how (\ref{lowerbound}) is derived, we use the triangle inequalities (Remark \ref{remark3}) and (\ref{delta_min}) to obtain:
\begin{eqnarray}
\label{lowerbound2'}
&&\sup_{t\ge T}\mathbb P\left(\min_{j\in\{1,\ldots,N(t)\}}\gamma_j^t<\frac{\delta_{\min}}{2}\right)\nonumber\\
&&\le\sup_{t\ge T}\mathbb P\left(\min_{\substack{j\in\{1,\ldots,N(t)\}\\ k,k'\in\{1,\ldots,\kappa\}\\ k\neq k'}}\left(d\left(X^{(k)}, X^{(k')}\right) -2\widehat{d^*}\left(\mathbf{x}_{c_k^j}^t, X^{(k)}\right)\right)<\frac{\delta_{\min}}{2}\right)\nonumber\\
&&\le\sup_{t\ge T}\mathbb P\left(\max_{\substack{j\in\{1,\ldots,N(t)\}\\ k\in\{1,\ldots,\kappa\}}} \widehat{d^*}\left(\mathbf{x}_{c_k^{j}}^t, X^{(k)}  \right)>\frac{\delta_{\min}}{4}\right).
\end{eqnarray}
Since the clusters are ordered in the order of appearance of the distinct process covariance structures, we have $\mathbf{x}_{c_k^j}^t = \mathbf{x}_{s_{k}}^t$ for all $j\ge m$ and $k \in\{ 1,\ldots,\kappa\}$, where we recall that the index $s_{k}$ is defined in (\ref{sk}). It follows from (\ref{lowerbound2'}), the fact that $\varepsilon<\delta_{\min}/4$  and  (\ref{d:upperbound}) that
\begin{equation}
\label{lowerbound2}
\sup_{t\ge T}\mathbb P\left(\min_{j\in\{1,\ldots,N(t)\}}\gamma_j^t<\frac{\delta_{\min}}{2}\right)\le\sup_{t\ge T}\mathbb P\left(\max_{\substack{j\in\{1,\ldots,m\} \\ k\in\{1,\ldots,\kappa\}}} \widehat{d^*}\left(\mathbf{x}_{c_k^{j}}^t, X^{(k)}\right)>\varepsilon\right)< m\delta.
\end{equation}
For $j\in\{ 1,\ldots,N(t)\}$, by (\ref{eta_gamma}), the triangle inequality, (\ref{deltamax:kk'}) and (\ref{lowerbound2}), we have
\begin{eqnarray}
\label{def:M}
&&\sup_{t\ge T}\mathbb P\left(\max_{j\in\{1,\ldots,N(t)\}}\gamma_j^t> \delta_{\max} + 2\varepsilon\right)\nonumber\\
&&\le \sup_{t\ge T}\mathbb P\left(\max_{\substack{j\in\{1,\ldots,N(t)\} \\  k,k'\in\{1,\ldots,\kappa\}\\ k\neq k'}}\left(d\left(X^{(k)}, X^{(k')}\right) +2\widehat{d^*}\left(\mathbf{x}_{c_k^{j}}^t, X^{(k)}\right) \right)> \delta_{\max} + 2\varepsilon\right)\nonumber \\
&&\le \sup_{t\ge T}\mathbb P\left(\max_{\substack{ j\in\{1,\ldots,m\}\\  k\in\{1,\ldots,\kappa\}\\}}\widehat{d^*}\left(\mathbf{x}_{c_k^{j}}^t, X^{(k)} \right)> \varepsilon\right)\nonumber \\
&&<m\delta.
\end{eqnarray}
\item[Upper bound of $\eta^t$:]
By (\ref{lowerbound2}) and the fact that
$
\sum_{j=1}^{N(t)}w_j\ge w_m,
$
we have
\begin{eqnarray}
\label{etabound}
&&\sup_{t\ge T}\mathbb P\left(\eta^t < \frac{w_{m} \delta_{\min}}{2}\right)\le \sup_{t\ge T}\mathbb P\left(\min_{j\in\{1,\ldots,N(t)\}}\gamma_{j}^t\sum_{j=1}^{N(t)}w_j < \frac{w_{m} \delta_{\min}}{2}\right)\nonumber\\
&&\le \sup_{t\ge T}\mathbb P\left(\min_{j\in\{1,\ldots,N(t)\}}\gamma_{j}^t< \frac{\delta_{\min}}{2}\right)<m\delta.
\end{eqnarray}
\end{description}
Recall that $N(t)>J$ for $t\ge T$. Therefore for every $k \in \{1,\ldots,\kappa\}$ we can write
\begin{eqnarray}
\label{split}
&&\frac{1}{\eta^t}\sum_{j=1}^{N(t)} w_j \gamma_j^t\widehat{d^*}\left(\mathbf{x}_{c_k^j}^t ,X^{(k)} \right)=\frac{1}{\eta^t}\sum_{j=1}^{m-1} w_j \gamma_j^t\widehat{d^*}\left(\mathbf{x}_{c_k^j}^t ,X^{(k)} \right)\nonumber\\
&&\hspace{2cm}+\frac{1}{\eta^t}\sum_{j=m}^{J} w_j \gamma_j^t\widehat{d^*}\left(\mathbf{x}_{c_k^j}^t ,X^{(k)} \right)+\frac{1}{\eta^t}\sum_{j=J+1}^{N(t)} w_j \gamma_j^t\widehat{d^*}\left(\mathbf{x}_{c_k^j}^t ,X^{(k)} \right).
\end{eqnarray}
Now we provide upper bounds in probability of the 3 terms on the right-hand side of (\ref{split}).
\begin{description}
\item[Upper bound of the first term:] In view of (\ref{d:upperbound}) and the fact that
$
(\eta^t)^{-1}\sum_{j=1}^{m-1} w_j \gamma_j^t\le 1,
$
we get
\begin{eqnarray}
\label{upper2}
&&\sup_{t\ge T}\mathbb P\left(\frac{1}{\eta^t}\sum_{j=1}^{m-1} w_j \gamma_j^t\widehat{d^*}\left(\mathbf{x}_{c_k^j}^t ,X^{(k)} \right)>\varepsilon\right)\nonumber\\
&&\le \sup_{t\ge T}\mathbb P\left(\max_{\substack{j\in\{1,\ldots,m-1\}\\ k\in\{1,\ldots,\kappa\}}}\widehat{d^*} \left(\mathbf{x}_{c_k^j}^t ,X^{(k)} \right)> \varepsilon\right)\nonumber\\
&&\le (m-1)\delta.
\end{eqnarray}
\item[Upper bound of the second term:] Recall that $\mathbf{x}_{c_k^j}^t = \mathbf{x}_{s_{k}}^t$ for all $j \in \{m,\ldots,J\}$ and $k \in\{ 1,\ldots,\kappa\}$. Therefore, by (\ref{d:upperbound}) and the fact that
$
(\eta^t)^{-1}\sum_{j=m}^{J} w_j \gamma_j^t\le 1,
$
for every $k\in\{1,\ldots,\kappa\}$ we have
\begin{eqnarray}
\label{upper3}
&&\sup_{t\ge T}\mathbb P\left(\frac{1}{\eta^t} \displaystyle \sum_{j=m}^{J}w_j \gamma_j^t \widehat{d^*}\left(\mathbf{x}_{c_k^j}^t,X^{(k)}\right)>\varepsilon\right)\nonumber\\
&&=  \sup_{t\ge T}\mathbb P\left(\widehat{d^*} \left(\mathbf{x}_{s_{k}}^t ,X^{(k)}\right) \frac{1}{\eta^t}\sum_{j=m}^{J} w_j \gamma_j^t>\varepsilon\right)\nonumber\\
&&\le \sup_{t\ge T}\mathbb P\left(\widehat{d^*} \left(\mathbf{x}_{s_{k}}^t ,X^{(k)}\right)>\varepsilon\right)<\delta.
\end{eqnarray}
\item[Upper bound of the third term:] By (\ref{bound_wJ}), (\ref{etabound}) and (\ref{def:M}),
\begin{eqnarray}
\label{upper1}
&&\sup_{t\ge T}\mathbb P\left(\frac{1}{\eta^t}\sum_{j=J+1}^{N(t)} w_j \gamma_j^t\widehat{d^*}\left(\mathbf{x}_{c_k^j}^t ,X^{(k)} \right)>\frac{2\varepsilon^2(\delta_{\max}+2\varepsilon)}{w_{m}\delta_{\min}}\right)\nonumber\\
&&\le \sup_{t\ge T}\mathbb P\left(\max_{\substack{j\in\{1,\ldots,N(t)\}\\ k\in\{1,\ldots,\kappa\}}}\widehat{d^*} \left(\mathbf{x}_{c_k^j}^t ,X^{(k)} \right)> \varepsilon\right)\nonumber\\
&&\hspace{2cm}+ \sup_{t\ge T}\mathbb P\left(\max_{j\in\{1,\ldots,N(t)\}}\gamma_j^t> \delta_{\max}+2\varepsilon\right)\nonumber\\
&&<2m\delta.
\end{eqnarray}
\end{description}
Combining (\ref{split}), (\ref{upper2}), (\ref{upper3}) and  (\ref{upper1}) we obtain, for $k \in \{1,\ldots,\kappa\}$,
\begin{eqnarray}
\label{conv_proof}
&&\sup_{t\ge T}\mathbb P\left(\frac{1}{\eta^t} \displaystyle \sum_{j=1}^{N(t)} w_j \gamma_j^t \widehat{d^*}\left(\mathbf{x}_{c_k^j}^t ,X^{(k)}\right)> \varepsilon \left(2 + \frac{2\varepsilon^2(\delta_{\max}+2\varepsilon)}{w_{m}\delta_{\min}}\right)\right)<3m\delta.
\end{eqnarray}
Now we explain how to use (\ref{conv_proof}) to prove the asymptotic consistency of Algorithm \ref{algo::online_known_k}. Consider an index $i \in G_{k'}$ for some $k' \in \{1,\ldots,\kappa\}$. On one hand, using the triangle inequalities, we get for $k\in\{1,\ldots,\kappa\}$, $k\neq k'$,
\begin{eqnarray*}
&&\frac{1}{\eta^t} \sum_{j=1}^{N(t)} w_j \gamma_j^t \widehat{d^*}\left(\mathbf{x}_i^t,\mathbf{x}_{c_k^j}^t\right)\geq \frac{1}{\eta^t} \sum_{j=1}^{N(t)} w_j \gamma_j^t d^*\left(X^{(k)},X^{(k')}\right)\\
&&\hspace{2cm}- \left(\frac{1}{\eta^t}\sum_{j=1}^{N(t)} w_j \gamma_j^t\right)\widehat{d^*}\left(\mathbf{x}_i^t, X^{(k')} \right)- \frac{1}{\eta^t} \sum_{j=1}^{N(t)} w_j \gamma_j^t \widehat{d^*}\left(\mathbf{x}_{c_k^j}^t ,X^{(k)}\right)\\
&&\ge \delta_{\min}-\widehat{d^*}\left(\mathbf{x}_i^t, X^{(k')} \right)+\frac{1}{\eta^t}\sum_{j=1}^{N(t)} w_j \gamma_j^t\widehat{d^*}\left(\mathbf{x}_{c_k^j}^t ,X^{(k)}\right).
\end{eqnarray*}
Then applying (\ref{d:upperbound}) and (\ref{conv_proof}) we obtain
\begin{eqnarray}
\label{lower:dhat}
&&\sup_{t\ge T}\mathbb P\left(\frac{1}{\eta^t} \sum_{j=1}^{N(t)} w_j \gamma_j^t \widehat{d^*}\left(\mathbf{x}_i^t,\mathbf{x}_{c_k^j}^t\right)<\delta_{\min} - \varepsilon \left(3 + \frac{2\varepsilon^2(\delta_{\max}+2\varepsilon)}{w_{m}\delta_{\min}}\right)\right)\nonumber\\
&&\le \sup_{t\ge T}\mathbb P\left(\widehat{d^*}\left(\mathbf{x}_i^t, X^{(k')} \right)+\frac{1}{\eta^t}\sum_{j=1}^{N(t)} w_j \gamma_j^t\widehat{d^*}\left(\mathbf{x}_{c_k^j}^t ,X^{(k)}\right)>\varepsilon \left(3 + \frac{2\varepsilon^2(\delta_{\max}+2\varepsilon)}{w_{m}\delta_{\min}}\right)\right)\nonumber\\
&&\le \sup_{t\ge T}\mathbb P\left(\widehat{d^*}\left(\mathbf{x}_i^t, X^{(k')} \right)>\varepsilon\right)\nonumber\\
&&\hspace{2cm}+\sup_{t\ge T}\mathbb P\left(\frac{1}{\eta^t}\sum_{j=1}^{N(t)} w_j \gamma_j^t\widehat{d^*}\left(\mathbf{x}_{c_k^j}^t ,X^{(k)}\right)>\varepsilon \left(2 + \frac{2\varepsilon^2(\delta_{\max}+2\varepsilon)}{w_{m}\delta_{\min}}\right)\right)\nonumber\\
&&<(3m+1)\delta.
\end{eqnarray}
On the other hand, from (\ref{d:upperbound}) we see for any $N\ge1$,
\begin{align}
\label{left_bound}
 \sup_{t\ge T_1(N)}\mathbb P\left(\max_{\substack{k \in \{1,\ldots,\kappa\} \\ i \in G_k \cap \left \{ 1,\ldots,N \right \}} } \widehat{d^*}\left(\mathbf{x}_i^t, X^{(k)}\right) >\varepsilon\right)<\delta.
\end{align}
Using again the triangle inequalities, we get
\begin{eqnarray}
\label{ineq_triangle_upper}
&&\frac{1}{\eta^t} \sum_{j=1}^{N(t)} w_j \gamma_j^t \widehat{d^*}\left(\mathbf{x}_i^t,\mathbf{x}_{c_{k'}^j}^t\right)\nonumber\\
&&\le \left(\frac{1}{\eta^t}\sum_{j=1}^{N(t)} w_j \gamma_j^t\right)\widehat{d^*}\left(\mathbf{x}_i^t, X^{(k')} \right)+ \frac{1}{\eta^t} \sum_{j=1}^{N(t)} w_j \gamma_j^t \widehat{d^*}\left(\mathbf{x}_{c_{k'}^j}^t ,X^{(k')}\right)\nonumber\\
&&\le \widehat{d^*}\left(\mathbf{x}_i^t, X^{(k')} \right)+\frac{1}{\eta^t}\sum_{j=1}^{N(t)} w_j \gamma_j^t\widehat{d^*}\left(\mathbf{x}_{c_{k'}^j}^t ,X^{(k')}\right).
\end{eqnarray}
Let $T':=\max\{T,T_1(N)\}$. It results from (\ref{ineq_triangle_upper}), (\ref{left_bound}) and (\ref{conv_proof}) that
\begin{eqnarray}
\label{upper:dhat}
&&\sup_{t\ge T'}\mathbb P\left(\frac{1}{\eta^t} \sum_{j=1}^{N(t)} w_j \gamma_j^t \widehat{d^*}\left(\mathbf{x}_i^t,\mathbf{x}_{c_{k'}^j}^t\right)>\varepsilon \left(3 + \frac{2\varepsilon^2(\delta_{\max}+2\varepsilon)}{w_{m}\delta_{\min}}\right)\right)\nonumber\\
&&\le \sup_{t\ge T'}\mathbb P\left(\widehat{d^*}\left(\mathbf{x}_i^t, X^{(k')} \right)+\frac{1}{\eta^t}\sum_{j=1}^{N(t)} w_j \gamma_j^t\widehat{d^*}\left(\mathbf{x}_{c_{k'}^j}^t ,X^{(k')}\right)>\varepsilon \left(3 + \frac{2\varepsilon^2(\delta_{\max}+2\varepsilon)}{w_{m}\delta_{\min}}\right)\right)\nonumber\\
&&\le \sup_{t\ge T'}\mathbb P\left(\widehat{d^*}\left(\mathbf{x}_i^t, X^{(k')} \right)>\varepsilon\right)\nonumber\\
&&\hspace{2cm}+\sup_{t\ge T'}\mathbb P\left(\frac{1}{\eta^t}\sum_{j=1}^{N(t)} w_j \gamma_j^t\widehat{d^*}\left(\mathbf{x}_{c_{k'}^j}^t ,X^{(k')}\right)>\varepsilon \left(2 + \frac{2\varepsilon^2(\delta_{\max}+2\varepsilon)}{w_{m}\delta_{\min}}\right)\right)\nonumber\\
&&<(3m+1)\delta.
\end{eqnarray}
Since $\delta$ and $\varepsilon$ can be chosen arbitrarily small, it follows from (\ref{lower:dhat}) and (\ref{upper:dhat}) that
\begin{align}
\label{final:bound}
\mathbb P\left(\argmin_{k \in \{1,\ldots,\kappa\}} \frac{1}{\eta^t}  \sum_{j=1}^{N(t)} w_j \gamma_j \widehat{d^*}\left(\mathbf{x}_i^t ,\mathbf{x}_{c_k^j}^t\right) = k'\right)\xrightarrow[t\to+\infty]{}1,
\end{align}
for all $i\in\{1,\ldots,N\}$. (\ref{cluster:x_1}) as well as Theorem \ref{thm:online} is proved.
\end{proof}

\section{Tests on Simulated Data: Clustering Multifractional Brownian Motions} \label{sec::exper_results}
\subsection{Efficiency Improvement: $\log^*$-transformation}
In this section, we show performance of the proposed clustering approaches (Algorithm \ref{algo::offline_known_k}) and (Algorithm \ref{algo::online_known_k}) on clustering simulated multifractional Brownian motions (mBm). MBm is a paradigmatic example of locally asymptotically self-similar processes. Its tangent process is fractional Brownian motion (fBm), which is self-similar. Since the covariance structure of the fBm is nonlinearly dependent on its self-similarity index, we can then apply the so-called $\log^*$-transformation to the covariance matrices of its increments, in order to improve the efficiency of the clustering algorithms. More precisely, in our clustering algorithms, we replace in $\widehat{d^*}$ the coefficients of all the covariance matrices and their estimators with their $\log^*$-transformation, i.e., for $x\in\mathbb R$,
\begin{equation}
\label{log}
\log^*(x):=\left\{\begin{array}{ll}
\log x,&\mbox{if $x>0$};\\
-\log(-x),&\mbox{if $x<0$};\\
0,&\mbox{if $x=0$}.
\end{array}\right.
\end{equation}
By applying such transformation, the observations assigned to any two clusters by the covariance structure ground truths become well separated thus the clustering algorithms become more efficient. For more detail on this efficiency improvement approach we refer the readers to  Section 3 in \cite{PRZ19}.

\subsection{Simulation Methodology}
Recall that an mBm $\{W_{H(t)}(t)\}_{t\ge0}$ is a zero-mean continuous-time Gaussian process, which can be defined via its covariance function \citep{Benassi1997,ACLV00,Stoev2006}: for $s,t\ge0$,
\begin{eqnarray}
\label{eqn::mbm_cov}
& \mathbb Cov\left(W_{H(t)}(t), W_{H(s)}(s)\right) := D(H(t), H(s))  \notag \\
& \quad \quad \times\left( t^{H(t)+H(s)} + s^{H(t)+H(s)} - |t-s|^{H(t)+H(s)} \right),
\end{eqnarray}
where
\begin{equation*}
D(t,s) := \frac{\sqrt{\Gamma(2t+1)\Gamma(2s+1)\sin(\pi t)\sin(\pi s)}}{2\Gamma(t+s+1)\sin(\pi(t+s)/2)}.
\end{equation*}
It can be seen from \cite{Boufoussi2008} that the mBm is locally asymptotically self-similar satisfying Assumption $(\mathcal A)$. Its tangent process at $t$ is an fBm $\{B^{(H(t))}(u)\}_u$ with index $H(t)$:
\begin{equation}
\label{mBm:conv}
\left\{\frac{W_{H(t+\tau u)}(t+\tau u) - W_{H(t)}(t)}{\tau^{H(t)}}\right\}_{u}\xrightarrow[\tau\rightarrow0^{+}]{\mbox{f.d.d.}}C_{H(t)}\left\{ B^{(H(t))}(u)\right\}_{u},
\end{equation}
where $C_{H(t)}$ is a deterministic function only depending on $H(t)$.

We select Wood-Chan's simulation method \citep{WoodChan,Chan1998simulation} to simulate the mBm paths, and use the implementation (MATLAB) of Wood-Chan's method in \textit{FracLab} (version 2.2) by INRIA in our simulation study\footnote{\url{https://project.inria.fr/fraclab/download/overview/}.}. This method outputs independent sample paths of the following form:
\begin{equation}
\label{sample:mBm}
\left\{W_{H(i/n)}\left(\frac{i}{n}\right)\right\},~\mbox{for $i=0,1,\ldots,n$},
\end{equation}
where $n\ge1$ is an input parameter. Now we would select $w_j=1/(j^2(j+1)^2)$ so that $d$ (see (\ref{def:d*})) is a convergent series (well-defined). To show this choice is reasonable we consider the stochastic process
\begin{equation}
\label{sample:mBm1}
\left\{W_{H(i\Delta)}\left(i\Delta\right)\right\},~\mbox{for $i=0,1,\ldots$},
\end{equation}
where $\Delta>0$ is some given mesh.
For each $t_0=0,\Delta,2\Delta,\ldots$, the increments of the tangent process (see the right-hand side of (\ref{local_self_2}))
$$
\left\{C_{H(t_0)}n^{-H(t_0)} B^{(H(t_0))}\left(i\Delta\right)\right\}, ~\mbox{for $i=0,1,\ldots$}
$$
is given by: for $i=0,1,\ldots$,
$$
X^{(H(t_0))}\left(i\Delta\right)=C_{H(t_0)}n^{-H(t_0)}\left( B^{(H(t_0))}\left((i+1)\Delta\right)-B^{(H(t_0))}\left(i\Delta\right)\right).
$$
As increments of fBm, $X^{(H(t_0))}(\bullet)$ is autocovariance ergodic. Moreover for $i,j=0,1,\ldots$, we have, by using the definition of $\log^*$ (see (\ref{log})), the covariance function of fBm (see (\ref{self-similar-cov})) and the fact that $\sup_{s\ge0}H(s)\le 1$,
\begin{eqnarray}
\label{cov:fbm}
&&\log^*\left(\mathbb Cov\left(X^{H(t_0))}\left(i\Delta\right),X^{(H(t_0))}\left(j\Delta\right)\right)\right)\nonumber\\
&&=\log^*\left(\frac{C_{H(t_0)}^2\Delta^{2H(t_0)}}{2}\left(\left|i-j-1\right|^{2H(t_0)}+\left|i-j+1\right|^{2H(t_0)}-2\left|i-j\right|^{2H(t_0)}\right)\right)\nonumber\\
&&=\mathcal O\left(\log (|i-j|+1)\right),~\mbox{as $|i-j|\to+\infty$}.
\end{eqnarray}
From the definition of $d$ in (\ref{def:d*}) and (\ref{cov:fbm}) we can see that, by taking $w_j=1/(j^2(j+1)^2)$ and using (\ref{cov:fbm}), for any $t_0,t_0'=0,\Delta,2\Delta,\ldots$,
$$
d\left(X^{(H(t_0))},X^{(H(t_0'))}\right)=\mathcal O\left(\sum_{l,m=1}^{+\infty}\frac{\sum_{i,j=l}^{l+m-1}\log (|i-j|+1)}{l^2(l+1)^2m^2(m+1)^2}\right),
$$
where,
\begin{eqnarray*}
&&\sum_{l,m=1}^{+\infty}\frac{\sum_{i,j=l}^{l+m-1}\log (|i-j|+1)}{l^2(l+1)^2m^2(m+1)^2}= 2\sum_{l,m=1}^{+\infty}\frac{\sum_{k=1}^{m-1}(m-k)\log (k+1)}{l^2(l+1)^2m^2(m+1)^2}\nonumber\\
&&\le 2\sum_{l,m=1}^{+\infty}\frac{\log m}{l^2(l+1)^2(m+1)^2}<+\infty.
\end{eqnarray*}
Therefore we have shown that $w_j=1/(j^2(j+1)^2)$ leads to that $d$ is well-defined.

\subsection{Synthetic Datasets}
To construct a collection of the mBm paths with distinct functional indexes $H(\bullet)$, we set the function form of $H(\bullet)$ in each of the predetermined clusters. Two functional forms of $H(\bullet)$ are selected for synthetic data study:
\begin{itemize}
    \item \textbf{Case 1 (Monotonic function)}: The general form is taken to be
\begin{equation} \label{eqn::mono}
H(t) = 0.5 + h \cdot t/Q, \quad  t\in[0,Q],
\end{equation}
where $Q>0$ is a fixed integer and different values of $h$ correspond to difference clusters.
We then predetermine 5 clusters with various $h$'s to separate different clusters. In this study we set $Q=100$, $h_1 = -0.4,~h_2 = -0.2,~h_3 = 0,~h_4 = 0.2$ and $h_5 = 0.4$. The trajectories of the 5 functional forms of $H(\bullet)$ in different clusters are illustrated in the top graph of Figure \ref{fig::sim_results_mono}.

    \item \textbf{Case 2 (Periodic function)}: The general form is taken to be
\begin{equation}  \label{eqn::sin}
H(t) = 0.5 + h \cdot \sin(\pi t/Q), \quad  t\in[0,Q],
\end{equation}
where  different values of $h$ lead to different clusters. Specifically,  we take $Q=100$, $h_1 = 0.4,~h_2 = 0.2,~h_3 = 0,~h_4 = -0.2$ and $h_5 = -0.4$. The trajectories of the corresponding 5 functional forms of $H(\bullet)$ are illustrated in the top graph of Figure \ref{fig::sim_results_sin}.
\end{itemize}

We demonstrate the approximated asymptotic consistency of the proposed algorithms by conducting both offline and online clustering analysis. Denote the number of observed data points in each time series by $n(t)$, and denote the number of time series paths by $N(t)$.

Under offline setting, the number of observed paths does not depend on time $t$, owever the lengths do. In order to construct offline datasets, we perform the follows:
\begin{enumerate}
    \item For $i=1,\ldots,5$, simulate 20 mBm paths in group $i$ (corresponding to $h_i$), each path is with length of $305$. Then the total number of paths $N=100$. To be more explicit we denote by
    \begin{equation}
    \label{S:total}
    S:=\begin{pmatrix}
x_{1,1} & x_{1,2} & \cdots & x_{1,305} \\
x_{2,1} & x_{2,2} & \cdots & x_{2,305} \\
\vdots  & \vdots  & \ddots & \vdots  \\
x_{100,1} & x_{100,2} & \cdots & x_{100,305}
\end{pmatrix},
    \end{equation}
    where each row is an mBm discrete-time path. For $i=1,\ldots,5$, the data from the $i$th group are given as:
    \begin{equation}
    \label{S_i}
     S^{(i)}:=\begin{pmatrix}
x_{20(i-1)+1,1} & x_{20(i-1)+1,2} & \cdots & x_{20(i-1)+1,305} \\
\vdots  & \vdots  & \ddots & \vdots  \\
x_{20i,1} & x_{20i,2} & \cdots & x_{20i,305}
\end{pmatrix}.
    \end{equation}
    \item At each $t=1,\ldots,100$, we suppose to observe the first $n(t)=3t+5$ values of each path, i.e.,
    $$
    S_{\text{offline}}(t)=\begin{pmatrix}
x_{1,1} & x_{1,2} & \cdots & x_{1,3t+5} \\
x_{2,1} & x_{2,2} & \cdots & x_{2,3t+5} \\
\vdots  & \vdots  & \ddots & \vdots  \\
x_{100,1} & x_{100,2} & \cdots & x_{100,3t+5}
\end{pmatrix}.
    $$
\end{enumerate}

The online dataset does not require observed paths to be with equal length, and can be regarded as some extension of the offline case. Introducing the online dataset aims at mimicking the situation where new time series are observed as time goes. In our simulation study, we use the following way to construct online datasets:
\begin{enumerate}
    \item For $i=1,\ldots,5$, simulate 20 mBm paths in group $i$ (corresponding to $h_i$), each path is with length of $305$ (see (\ref{S:total}) and (\ref{S_i})).
    \item At each $t=1,\ldots,100$ and $i=1,\ldots,5$, we suppose to observe the following dataset in the $i$th group:
$$
 S^{(i)}_{\text{online}}(t)=\begin{pmatrix}
\tilde x_{1,1} & \tilde x_{1,2} & \cdots&\cdots&\cdots & \cdots &\tilde x_{1,n_1(t)} \\
\tilde x_{2,1} & \tilde x_{2,2} & \cdots&\cdots &\cdots & \tilde x_{1,n_2(t)} \\
\vdots  & \vdots  & \ddots & \vdots  \\
\tilde x_{N_i(t),1} & \tilde x_{N_i(t),2} & \cdots & \tilde x_{N_i(t),n_{N_i(t)}(t)}
\end{pmatrix},
$$
    where
    \begin{itemize}
    \item $\tilde x_{k,l}$'s are the $(k,l)$-coefficients in $S^{(i)}$ given in (\ref{S_i}).
        \item $N_i(t):=6 + \lfloor (t-1)/10 \rfloor$ denotes the number of paths in the $i$th group. Here $\lfloor\bullet\rfloor$ denotes the floor number. That is, starting from 6 paths in each group, $1$ new path will be added into each group as $t$ increases by $10$.
        \item $n_l(t): = 3\left(t-(l-6)^+\right)^+ + 5$, with $(\bullet)^+ := \max(\bullet,0)$. This means each path observes 3 new values as $t$ increases by 1.
    \end{itemize}
\end{enumerate}
Since at each time $t$, the covariance structure ground truth being known, we can then evaluate the clustering performance in terms of the so-called misclassification rates \citep{PRZ19}. Heuristically speaking, the misclassification rate is then calculated by averaging the proportion of mis-clustered paths in each scenario.

\subsection{Experimental Results}
We demonstrate the asymptotic consistency of our clustering algorithms by computing the misclassification rates using simulated offline and online datasets. More details about such misclassification rate are provided in Section 4 of \cite{PRZ19}.

Below we summarize the simulation study results.
\begin{description}
\item[Case 1  (Monotonic function):]
\end{description}
When $H(\bullet)$'s are chosen to be 5 monotonic functionals of the form \eqref{eqn::mono} (see the top graph in Figure \ref{fig::sim_results_mono}), the bottom graph in Figure \ref{fig::sim_results_mono} illustrates the behavior of the misclassification rates corresponding to Algorithm \ref{algo::offline_known_k} applied to offline data setting (solid line), and Algorithm \ref{algo::online_known_k} applied to online data setting (dashed line). From this result we observe the following:
\begin{description}
\item[(1)] Both algorithms attempt to be consistent in their circumstances, as the time $t$ increases, in the sense that the corresponding misclassification rates are decreasing to $0$.
\item[(2)] Clustering mBms are asymptotically equivalent to clustering their tangent processes' increments.
\item[(3)] The online algorithm seems to have an overall better performance: its misclassification rates are $5\%$-$10\%$ lower than that of offline algorithm. The reason may be that at early time steps the differences among the $H(\bullet)$'s are not significantly. Unlike the offline clustering algorithm, the online one is flexible enough to catch these small differences.
\end{description}
\begin{figure}[h]
\centering
\includegraphics[scale = 0.8]{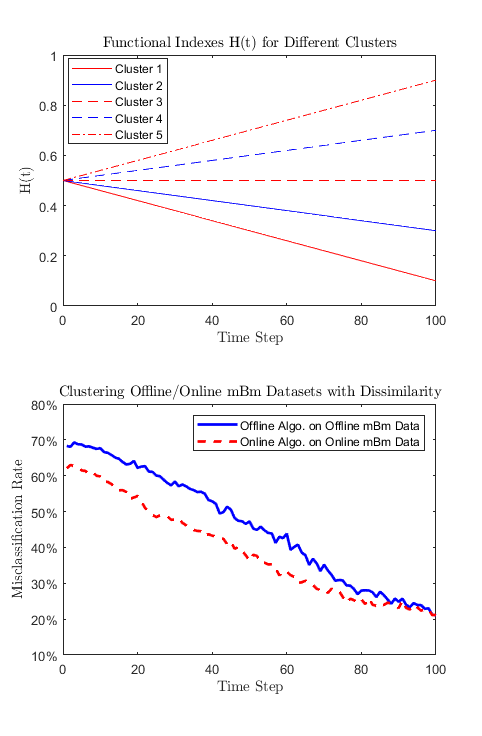}
\caption{The functional form of $H(\bullet)$ follows Eq. \eqref{eqn::mono}: $H(t) = 0.5 + h_i \cdot t/100$ with $t=0,1,\ldots,100$. The top graph plots $H(\bullet)$ corresponding to $5$ different clusters. The bottom graph illustrates the misclassification rates output by $(i)$ offline algorithm on offline dataset (solid line) and $(ii)$ online algorithm on online dataset (dashed line). Both algorithms are performed based on the $\log^*$- transformed covariance-based dissimilarity measure.}
\label{fig::sim_results_mono}
\end{figure}
\begin{description}
\item[Case 2  (Periodic function):]
\end{description}
The same converging behaviors are found in case of periodic functional form of $H(\bullet)$ as specified in \eqref{eqn::sin}. Their trajectories are illustrated in the top graph of Figure \ref{fig::sim_results_sin}. The clustering performance shown in the bottom graph of Figure \ref{fig::sim_results_sin} indicate the following:
\begin{description}
\item[(1)] Both misclassification rates of the clustering algorithms have generally a declining trend as time increases.
\item[(2)] As the differences among the periodic function $H(\bullet)$'s values go up and down, the misclassification rates go down and up accordingly.
\item[(3)] The online clustering algorithm has an overall worse performance than the offline one. This may be because starting from $t=20$ the differences among $H(\bullet)$'s become significantly large. In this situation offline clustering algorithm can better catch these differences, since it has larger sample size ($20$ paths in each group) than the online one.
\end{description}
Finally note that in the simulation study, for each pair of paths with length $n(t)$, we have taken $K=n(t)-2$ and $L=1$ in $\widehat{d^*}$, however any other value of $K$ could be taken. We have provided easily readable and editable MATLAB codes of the proposed algorithms and simulation study replications. All the codes used in this section can be found publicly online\footnote{\url{https://github.com/researchcoding/clustering_locally_asymptotically_self_similar_processes/}.}.
\begin{figure}[h]
\centering
\includegraphics[scale = 0.8]{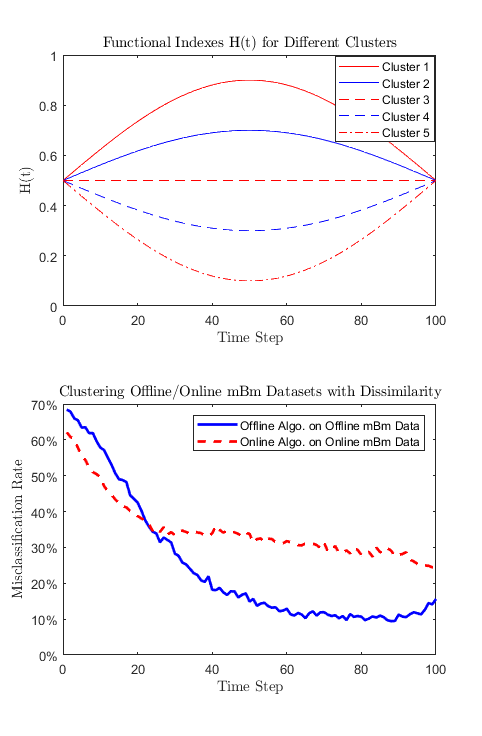}
\caption{The functional form of $H(\bullet)$ follows Eq. \eqref{eqn::sin}: $H(t) = 0.5 + h_i \cdot \sin(\pi t/100)$ with $t=0,1,\ldots,100$. The top graph plots $H(\bullet)$ corresponding to $5$ different clusters. The bottom graph illustrates the misclassification rates output by $(i)$ offline algorithm on offline dataset (solid line) and $(ii)$ online algorithm on online dataset (dashed line). Both algorithms are performed based on the $\log^*$- transformed covariance-based dissimilarity measure.}
\label{fig::sim_results_sin}
\end{figure}

\section{Real World Application: Clustering Global Financial Markets}
\label{sec:real_world}
\subsection{Motivation}
In this section, we motivate the application of the proposed clustering algorithms to real world datasets through performing cluster analysis on global equity markets. In the past, stock returns of countries in the same region are commonly believed to have more similar patterns and higher correlation. The reason is obvious: economic entities within closer geographical distance have potentially more trades and thus are influenced by similar economic factors. However, more recent empirical evidences of financial markets reveal that globalization is breaking the geographical barrier and is creating common economic factors. As a result, global economic clusters switch from ``geographical centriods'' to ``emerging/developed economics centriods''. That is, emerging markets demonstrate more and more similar financial market patterns and correlations \citep{DOYY18}, whereas developed economic entities share increasingly financial market similarity \citep{AL13}.

The idea of modeling stock returns by locally asymptotically self-similar processes (mBm) is pioneered by \cite{BP08}. This time-varying self-similar feature of the financial markets stochastic processes is further convinced by \cite{BPP13} and \cite{PZ18}. We consider data from global financial markets to be perfect underlying stochastic processes of our proposed clustering algorithms. We further examine the connection of global financial markets by answering whether economic entities are better co-behaved and clustered by geographical distribution or by development level.

\subsection{Data and Methodology}
Two asset classes from global financial markets are used in our empirical cluster analysis. The \textit{equity index return} captures the upside (growth) characteristics and the \textit{sovereign CDS spread} captures the downside (credit risk) characteristics of underlying economic entities:

\begin{itemize}
    \item \textbf{Equity indexes returns}: We cluster the global equity indexes according to the empirical time-varying covariance structure of their performance, using Algorithms \ref{algo::offline_known_k} and \ref{algo::online_known_k} as purposed in this paper. The index constituents of MSCI ACWI (All Country World Index), selected as underlying stochastic processes. Each of the indexes is a realized path representing the historical monthly total returns (with dividends) of underlying economic entities. MSCI ACWI is the leading global equity market index and covers about 85\% of the market capitalization in each market\footnote{As of December 2018, as reported on \url{https://www.msci.com/acwi}.}.

    \item \textbf{Sovereign CDS spreads}: We cluster the sovereign credit default swap (CDS) spreads of the same economic entities. The sovereign CDS is an insurance-like product that provides default protection on bonds and other debts issued by government, nation or large economic entity. Its spread reflects the cost to insurer the exposure to the possibility of a sovereign defaulting or restructuring. We select 5-year sovereign CDS spread as the indicator of sovereign credit risk, and 5-year products usually have the best liquidity on the market. The same economic entities as in equity index analysis are selected. Our data source is Bloomberg and Markit.
\end{itemize}

Through empirical study it is proved that these indexes returns exhibit the \enquote{long memory} path feature hence they can be modeled by self-similar processes or more generally by locally asymptotically self-similar processes such as fBms and mBms (see e.g. \cite{BP08} \cite{BPP13} and \cite{PZ18}). Therefore similar to Section \ref{sec::exper_results} we may cluster the increments of the indexes returns with the $\log^*$-transformed dissimilarity measure. We select equity indexes and sovereign CDS spreads covering 23 developed economic entities and 24 emerging markets. The detailed constituents can found at Table \ref{table::msci_data} or \url{https://www.msci.com/acwi}: Americas, EMEA (Europe, Middle East and Africa), Pacific and Asia.

\begin{sidewaystable}[htbp]
\centering
\caption{The categories of major equity and sovereign CDS markets in the MSCI ACWI (All Country World Index). There are 23 markets from developed economic entities, and 24 markets from emerging countries or areas. The geographical clustering contains Americas, EMEA (Europe, Middle East and Africa), Pacific and Asia. Markets with * are missing sovereign CDS data. \label{table::msci_data}}
\vspace{0.5em}
\begin{tabular}{ccccccc}
 \hline\hline
 \multicolumn{3}{c}{Developed Markets} & & \multicolumn{3}{c}{Emerging Markets} \\
 \cmidrule(r){1-3} \cmidrule(r){5-7}
Americas	&	Europe \& Middle East	&	Pacific	& &	Americas	&	Europe \& Middle East \& Africa	&	Asia	\\
 \cmidrule(r){1-3} \cmidrule(r){5-7}
Canada	&	Austria	&	Australia &	&	Brazil	&	Czech Republic	&	China (Mainland)	\\
USA	&	Belgium	&	 Hong Kong	& &	Chile	&	Greece*	&	India	\\
	&	Denmark	& Japan	&	& Colombia	&	Hungary	&	Indonesia	\\
	&	Finland	&	New Zealand &	&	Mexico	&	Poland	&	Korea	\\
	&	France	&	Singapore* &	&	Peru	&	Russia	&	Malaysia	\\
	&	Germany	&		& &		&	Turkey	&	Pakistan	\\
	&	Ireland	&		& &		&	Egypt	&	Philippines	\\
	&	Israel	&		& &		&	South Africa	&	Taiwan	\\
	&	Italy	&		& &		&	Qatar*	&	Thailand	\\
	&	Netherlands*	&		& &		&	United Arab Emirates*	&		\\
	&	Norway	&		& &		&		&		\\
	&	Portugal	&		& &		&		&		\\
	&	Spain	&		& &		&		&		\\
	&	Sweden	&		& &		&		&		\\
	&	Switzerland	&		& &		&		&		\\
	&	United Kingdom	&		& &		&		&		\\
 \hline \hline
\end{tabular}
\begin{flushleft}
\scriptsize\textbf{Source:} {MSCI ACWI (All Country World Index) market allocation. \url{https://www.msci.com/acwi}.}\\
\end{flushleft}
\end{sidewaystable}

We construct both offline and online datasets for monthly returns. For monthly return of global equity indexes, offline dataset starts from June 2005 and includes the financial crisis period in 2007 and 2008 and ends on November 2019. We include the global stock market crisis period to incorporate potential credit risk contagion effect on our clustering analysis. The online dataset starts on January 1989, which covers 1997 Asian financial crisis, 2003 dot-com bubble and 2007 subprime mortgage crisis, and ends on November 2019. In offline setting, each stochastic path in the dataset has 174 time series observations, and there are 47 paths to be clustered. In online setting, the longest time series have 371 observations and shortest time series have 174 observations. The online dataset begins with 33 economic entities and ends with 47 paths.

For monthly average on sovereign CDS spreads, we remove Netherlands, Qatar, Singapore, Greece and United Arabic from the economic entity samples due to insufficient observation. Therefore, we end up with 42 economic entities with CDS spread clustering analysis. The offline dateset starts from June 2005 and ends on November 2019, and each economic entity has 174 observations.

\subsection{Clustering Results}
We compare the clustering outcomes of both offline and online datasets with separations suggested by region (4 groups: Americas, Europe \& Middle East, Pacific and Asia) and development level (2 groups: emerging markets and developed markets). The clustering factor (region or development level) that has lower misclassification rate contributes the partition of the economics entities the most. That is, we examine whether region or development level differentiates the financial markets behaviors.

Table \ref{table::emp_results} shows that the misclassification rates for development levels are significantly and consistently lower than that of geographical region, for offline and online settings and for both equity and credit markets. The clustering comparison seems to convince that development level dominates the financial market characteristics over geographical distance for those underlying economic entities. The best results are clustering offline dataset using offline algorithm and clustering online dataset using online algorithm on dividing emerging markets and developed markets. There are 30\% to 50\% decrease on misclassification rates when clustering via development level than that via region.

\begin{table}[htbp]
\centering
\caption{The misclassification rates of clustering algorithms on datasets, comparing to clusters suggested by geographical region and development levels. Panel A presents the results from clustering equity indexes, and Panel B presents the results from clustering sovereign CDS spreads. \label{table::emp_results}}
\vspace{0.2em}
\begin{tabular}{ccccc}
 \hline\hline
\textbf{Panel A} & \multicolumn{2}{c}{Offline Algorithm} & \multicolumn{2}{c}{Online Algorithm} \\
 \cmidrule(r){2-3} \cmidrule(r){4-5}
Stock Returns	&	Regions	&	Emerging/Developed	&	Regions	&	Emerging/Developed	\\
\cmidrule(r){2-3} \cmidrule(r){4-5}
offline dataset	&	61.70\%	&	29.79\%	&	55.32\%	&	36.17\%	\\
online dataset	&	53.19\%	&	44.68\%	&	51.06\%	&	14.89\%	\\
 \hline \hline
\end{tabular}

\vspace{1cm}

\begin{tabular}{ccccc}
 \hline\hline
\textbf{Panel B} & \multicolumn{2}{c}{Offline Algorithm} & \multicolumn{2}{c}{Online Algorithm} \\
 \cmidrule(r){2-3} \cmidrule(r){4-5}
CDS Spreads	&	Regions	&	Emerging/Developed	&	Regions	&	Emerging/Developed	\\
\cmidrule(r){2-3} \cmidrule(r){4-5}
offline dataset	&	64.29\%	&	28.57\%	&	71.43\%	&	26.19\%	\\
online dataset	&	54.76\%	&	47.62\%	&	59.52\%	&	26.19\%	\\
 \hline \hline
\end{tabular}
\end{table}

Table \ref{table::emp_clu_results} presents misclassified economic entities from the cluster analysis on grouping emerging markets and developed markets. For equity indexes dataset, more misclassification concentrates on clustering developed economic entities into emerging market group. Stock index returns from Austria, Finland and Portugal markets are clustered as from emerging markets by both offline and online algorithms. The misclassification within developed group is rather random. For sovereign CDS spreads, more misclassification is on clustering emerging economic entities into developed market group. Sovereign CDS spreads of Chile, China (Mainland), Czech, Korea, Malaysia, Mexico, Poland and Thailand are consistently misclassified into developed markets, probably due to their low sovereign credit risk.

\begin{table}[htbp]
\centering
\caption{The misclassification outcome using offline algorithm on offline dataset and online algorithm on online dataset. Panel A reports the incorrectly categorized economics entities from equity index return clustering, and Panel B reports the incorrectly categorized economics entities from sovereign CDS spreads clustering. The algorithm divides the whole dataset into two groups: emerging market and developed markets, respectively. The incorrect outcome, where (i) entities from developed markets incorrectly clusters in emerging market, or (ii) vice versa, are reported in the table. \label{table::emp_clu_results}}
\vspace{0.2em}
Pabel A: Equity Indexes Returns
\begin{tabular}{cc|cc}
 \hline\hline
 \multicolumn{2}{c}{Group 1 (Emerging Markets)} & \multicolumn{2}{c}{Group 2 (Developed Markets)} \\
 \cmidrule(r){1-2} \cmidrule(r){3-4}
Incorrect - Offline	&	Incorrect - Online	&	Incorrect - Offline	&	Incorrect - Online	\\ \hline
Austria & Austria & Korea & Czech Republic \\
Finland &  Finland &  Chile & Qatar  \\
Germany &  Portugal & Philippines &  Peru  \\
Ireland &  &  Malaysia & South Africa \\
Italy & &  Mexico &  \\
Norway & &  & \\
Portugal & &  & \\
Spain & &  & \\
New Zealand & & & \\
 \hline \hline
\end{tabular}

\vspace{1cm}
Pabel B: Sovereign CDS Spreads
\begin{tabular}{cc|cc}
 \hline\hline
 \multicolumn{2}{c}{Group 1 (Emerging Markets)} & \multicolumn{2}{c}{Group 2 (Developed Markets)} \\
 \cmidrule(r){1-2} \cmidrule(r){3-4}
Incorrect - Offline	&	Incorrect - Online	&	Incorrect - Offline	&	Incorrect - Online	\\ \hline
Ireland & Ireland & Chile & Chile \\
Italy  & Portugal & China (Mainland)  & China (Mainland)  \\
Portugal & & Czech Republic &  Czech Republic   \\
Spain & & Korea & Hungary \\
&  & Malaysia &  Korea \\
 & & Mexico &  Malaysia \\
 & & Poland & Mexico\\
 & & Thailand & Poland \\
 & & & Thailand \\
 \hline \hline
\end{tabular}
\end{table}

From both equity and credit markets, we show that clustering financial time series using development level outperforms the clustering outcome using region. This empirical result further supports that economic globalization is breaking the geographic barrier and enhances the comovement of financial market behaviors by economics strengthens and status.

\section{Conclusions and Future Prospects} \label{conclusion}
We introduce the problem of clustering locally asymptotically self-similar processes. A new covariance-based dissimilarity measure is proposed to obtain approximately asymptotically consistent clustering algorithms for both offline and online settings. We have shown that the recommended algorithms are competitive for at least three reasons:
\begin{description}
\item[(1)] Given their flexibility, our algorithms are applicable to clustering any distribution stationary ergodic processes with finite variances; any autocovariance ergodic processes; locally asymptotically self-similar processes whose tangent processes have autocovariance ergodic increments. The multifractional Brownian motion (mBm) is an excellent example of the latter process.
\item[(2)] Our algorithms are efficient enough in terms of their computational complexity. Simulation study is performed on clustering mBm. The results show that both offline and online algorithms are approximately asymptotically consistent.
\item[(3)] Our algorithms are successfully applied to cluster the real world financial time series (equity returns and sovereign CDS spreads) via development level and via regions. The outcomes are self-consistent with the financial markets behavior.
\end{description}
Finally we list the following open problems which could be left for future research.
\begin{description}
\item[(1)] The clustering framework proposed in our paper only focuses on the cases where the true number of clusters $\kappa$ is known. The problem for which $\kappa$ is supposed to be unknown remains open.
\item[(2)] If we drop the Gaussianity assumption the class of stationary increments self-similar
processes becomes much larger. This will yield introduction to a more general class of locally asymptotically self-similar processes, whose autocovariances do not exist. This class includes linear multifractional stable motion \citep{Stoev2004stochastic, Stoev2005path} as a paradigmatic example. Cluster analysis on such stable processes will no-doubt lead to a wide range of applications, especially when the process distributions exhibit heavy-tailed phenomena. Neither the distribution dissimilarity measure introduced in \cite{khaleghi2016} nor the covariance-based dissimilarity measures used in this paper would work in this case, hence new techniques are required to cluster such processes.
\end{description}

\bibliographystyle{apalike}
\bibliography{ml.bib}

\end{document}